\documentclass[11pt,twoside]{article}
\usepackage{fullpage}

\usepackage{epsf}
\usepackage{fancyhdr}
\usepackage{graphics}
\usepackage{graphicx}
\usepackage{psfrag}
\usepackage{microtype}
\usepackage{algorithmic}
\usepackage{color,xcolor}
\usepackage[linesnumbered,ruled]{algorithm2e}
\DontPrintSemicolon
\usepackage{color}

\usepackage{amsthm}
\usepackage{amsfonts}
\usepackage{amsmath}
\usepackage{amssymb,bbm}

\usepackage{graphicx}
\usepackage{caption}
\usepackage{subcaption}

\usepackage{wrapfig}

\usepackage{float}

\usepackage{url}
\usepackage[colorlinks,linkcolor=magenta,citecolor=blue, pagebackref=true,backref=true]{hyperref}
\renewcommand*{\backrefalt}[4]{%
    \ifcase #1 \footnotesize{(Not cited.)}%
    \or        \footnotesize{(Cited on page~#2.)}%
    \else      \footnotesize{(Cited on pages~#2.)}%
    \fi}

\usepackage{nicefrac}
\usepackage{comment}

\usepackage{chngpage}

 \usepackage{tabularx}%

\usepackage{enumitem}
\usepackage{booktabs}
\usepackage{caption}

\usepackage{bm,bbm}
\usepackage{mathtools}

\newtheorem{assumption}{Assumption}

\newtheorem{proposition}{Proposition}
\newtheorem{definition}{Definition}

\newcommand{\norm}[1]{\left\lVert#1\right\rVert}

\def\TM{\mathcal{T}}

\def\TWA{W_{d_{\TM}}}

\def\OT{\textnormal{OT}}
\def\TW{\textnormal{TW}}
\def\TSW{\textnormal{TSW}}

\def\TDA{\textnormal{TDA}}
\def\PD{\textnormal{PD}}

\def\PF{\textnormal{PF}}

\def\PSS{\textnormal{PSS}}
\def\PWG{\textnormal{PWG}}
\def\SW{\textnormal{SW}}

\def\Dg{\textnormal{Dg}}

\def\RR{\mathbb{R}}

\def\NN{\mathbb{N}}

\def\Ff{\mathcal{F}}

\def\Pp{\mathcal{P}}

\def\Ff{\mathcal{F}}
\def\Rr{\mathcal{R}}

\def\Xx{\mathcal{X}}
\def\Tt{\mathcal{T}}

\begin{document}

\begin{center}
{\bf{\LARGE{Tree-Sliced Variants of Wasserstein Distances}}}

\vspace*{.2in}
    {\large{
 \begin{tabular}{cccc}
  Tam Le$^{\dagger}$ &  Makoto Yamada$^{\ddagger, \dagger}$ & Kenji Fukumizu$^{\diamond, \dagger}$ & Marco Cuturi$^{\triangleleft, \triangleright}$ \\
 \end{tabular}
 
}}

\vspace*{.2in}

 \begin{tabular}{c}
 RIKEN AIP$^\dagger$, Kyoto University$^\ddagger$ \\
 The Institute of Statistical Mathematics, Japan$^\diamond$ \\
 Google Brain, Paris$^\triangleleft$, CREST - ENSAE$^\triangleright$\\
 \end{tabular}
\vspace*{.2in}

\today

\vspace*{.2in}

\begin{abstract}

Optimal transport (\OT) theory defines a powerful set of tools to compare probability distributions. \OT~suffers however from a few drawbacks, computational and statistical, which have encouraged the proposal of several regularized variants of OT in the recent literature, one of the most notable being the \textit{sliced} formulation, which exploits the closed-form formula between univariate distributions by projecting high-dimensional measures onto random lines. We consider in this work a more general family of ground metrics, namely \textit{tree metrics}, which also yield fast closed-form computations and negative definite, and of which the sliced-Wasserstein distance is a particular case (the tree is a chain). We propose the tree-sliced Wasserstein distance, computed by averaging the Wasserstein distance between these measures using random tree metrics, built adaptively in either low or high-dimensional spaces. Exploiting the negative definiteness of that distance, we also propose a positive definite kernel, and test it against other baselines on a few benchmark tasks.

\end{abstract}
\end{center}

\section{Introduction}
\label{sec:introduction}

Many tasks in machine learning involve the comparison of two probability distributions, or histograms. Several geometries in the statistics and machine learning literature are used for that purpose, such as the Kullback-Leibler divergence, the Fisher information metric, the $\chi^2$ distance, or the Hellinger distance, to name a few. Among them, the optimal transport (\OT) geometry, also known as Wasserstein \cite{villani2008optimal}, Monge-Kantorovich \cite{kantorovich1942}, or Earth Mover's \cite{rubner2000}, has gained traction in the machine learning community~\cite{genevay2016stochastic, kusner2015word, LeeNIPS2018}, statistics \cite{ebert2017construction, panaretos2016amplitude}, or computer graphics \cite{lavenant2018dynamical, solomon2015convolutional}.




The naive computation of \OT~between two discrete measures involves solving a network flow problem whose computation scales typically cubically in the size of the measures~\cite{burkard1999}. There are two notable lines of work to reduce the time complexity of \OT. \textit{(i)} The first direction exploits the fact that simple ground costs can lead to faster computations. For instance, if one uses the binary metric $d(x,z)=\mathbbm{1}_{x\ne z}$ between two points $x,z$, the \OT~distance is equivalent to the total variation distance \cite[p.7]{villani2003topics}. When measures are supported on the real line $\RR$ and the cost $c$ is a nonegative convex function $g$ applied to the difference $z-x$ between two points, namely for $x, z \in \RR$, one has $c(x, z)=g(z-x)$, then the \OT~distance is equal to the integral of $g$ evaluated on the difference between the generalized quantile functions of these two probability distributions~\cite[\S2]{SantambrogioBook}. Other simplifications include thresholding the ground cost distance~\cite{pele2009fast} or considering for a ground cost the shortest-path metric on a graph~\cite[\S6]{PeyreCuturiBook}. \textit{(ii)} The second one is to use regularization to approximate solutions of \OT~problems, notably entropy~\cite{cuturi2013sinkhorn}, which results in a problem that can be solved using Sinkhorn iterations. Genevay et al. \cite{genevay2016stochastic} extended this approach to the semi-discrete and continuous \OT~problems using stochastic optimization. Different variants of Sinkhorn algorithm have been proposed recently~\cite{altschuler2017near,dvurechensky18}, and speed-ups are obtained when the ground cost is the quadratic Euclidean distance~\cite{altschuler2018approximating, altschuler2018massively}, or more generally the heat kernel on geometric domains~\cite{solomon2015convolutional}. The convergence of Sinkhorn algorithm has been considered in \cite{altschuler2017near, franklin1989scaling}.


In this work, we follow the first direction to provide a fast computation for \OT. To do so, we consider tree metrics as ground costs for \OT, which results in the so-called tree-Wasserstein (\TW) distance \cite{do2011sublinear, evans2012phylogenetic, mcgregor2013sketching}.  We consider two practical procedures to sample tree metrics based on spatial information for both low-dimensional and high-dimensional spaces of supports. Using these random tree-metrics, we propose tree-sliced-Wasserstein distances, obtained by averaging over several \TW~distances with various ground tree metrics. The \TW~distance, as well as its average over several trees, can be shown to be negative definite\footnote{In general, Wasserstein spaces are not Hilbertian~\cite[\S8.3]{PeyreCuturiBook}.}.  As a consequence, we propose a positive definite tree-(sliced-)Wasserstein kernel that generalizes the sliced-Wasserstein kernel~\cite{carriere17asliced, kolouri2016sliced}. 

The paper is organized as follows: we give reminders on \OT~and tree metrics in Section~\ref{sec:reminders}, introduce \TW~distance and its properties in Section~\ref{sec:TWD}, describe tree-sliced-Wasserstein variants with practical families of tree metrics, and proposed tree-(sliced)-Wasserstein kernel in Section~\ref{sec:TSW}, provide connections of \TW~with other work in Section~\ref{sec:relations}, and follow with experimental results on many benchmark datasets in word embedding-based document classification and topological data analysis in Section~\ref{sec:experiments}, before concluding in Section~\ref{sec:conclusion}. We have released code for these tools\footnote{\url{https://github.com/lttam/TreeWasserstein}.}.

\section{Reminders on Optimal Transport and Tree Metrics}
\label{sec:reminders}

In this section, we briefly review definitions of optimal transport (\OT) and tree metrics. Let $\Omega$ be a measurable space endowed with a metric $d$. For any $x \in \Omega$, we write $\delta_x$ the Dirac unit mass on $x$.

\paragraph{Optimal transport.} Let $\mu$, $\nu$ be two Borel probability distributions on $\Omega$, $\Rr(\mu, \nu)$ be the set of probability distributions $\pi$ on the product space $\Omega \times \Omega$ such that $\pi(A \times \Omega) = \mu(A)$ and $\pi(\Omega \times B) = \nu(B)$ for all Borel sets $A$, $B$. The 1-Wasserstein distance $W_d$ \cite[p.2]{villani2003topics} between $\mu$, $\nu$ is defined as:
\begin{equation}\label{equ:OTprob}
W_d(\mu, \nu) = \inf\left\{ \int_{\Omega \times \Omega} d(x, z) \pi(dx, dz) \mid \pi \in \Rr(\mu, \nu) \right\}.
\end{equation}
Let $\Ff_d$ be the set of Lipschitz functions w.r.t. $d$, i.e. functions $f : \Omega \rightarrow \RR$ such that $\left| f(x) - f(z)\right| \le d(x, z), \forall x, z \in \Omega$. The dual of ~\eqref{equ:OTprob} simplifies to the following problem \OT~\cite[Theorem 1.3, p.19]{villani2003topics} is:
\begin{equation}\label{equ:KantorovichDualityProb}
W_d(\mu, \nu) = \sup\left\{\int_{\Omega} f(x) \mu(dx) - \int_{\Omega} f(z) \nu(dz) \mid f \in \Ff_d \right\}.
\end{equation}


\paragraph{Tree metrics.}
A metric $d:\Omega\times\Omega\rightarrow \mathbf{R}$ is called a \textit{tree metric} on $\Omega$ if there exists a tree $\Tt$ with non-negative edge lengths such that all elements of $\Omega$ are contained in its nodes and such that for every $x, z \in \Omega$, one has that $d(x, z)$ equals to the length of the (unique) path between $x$ and $z$~\cite[\S 7, p.145--182]{semple2003phylogenetics}. We write the tree metric corresponding to that tree $d_{\TM}$.

\begin{figure}
  \begin{center}
    \includegraphics[width=0.24\textwidth]{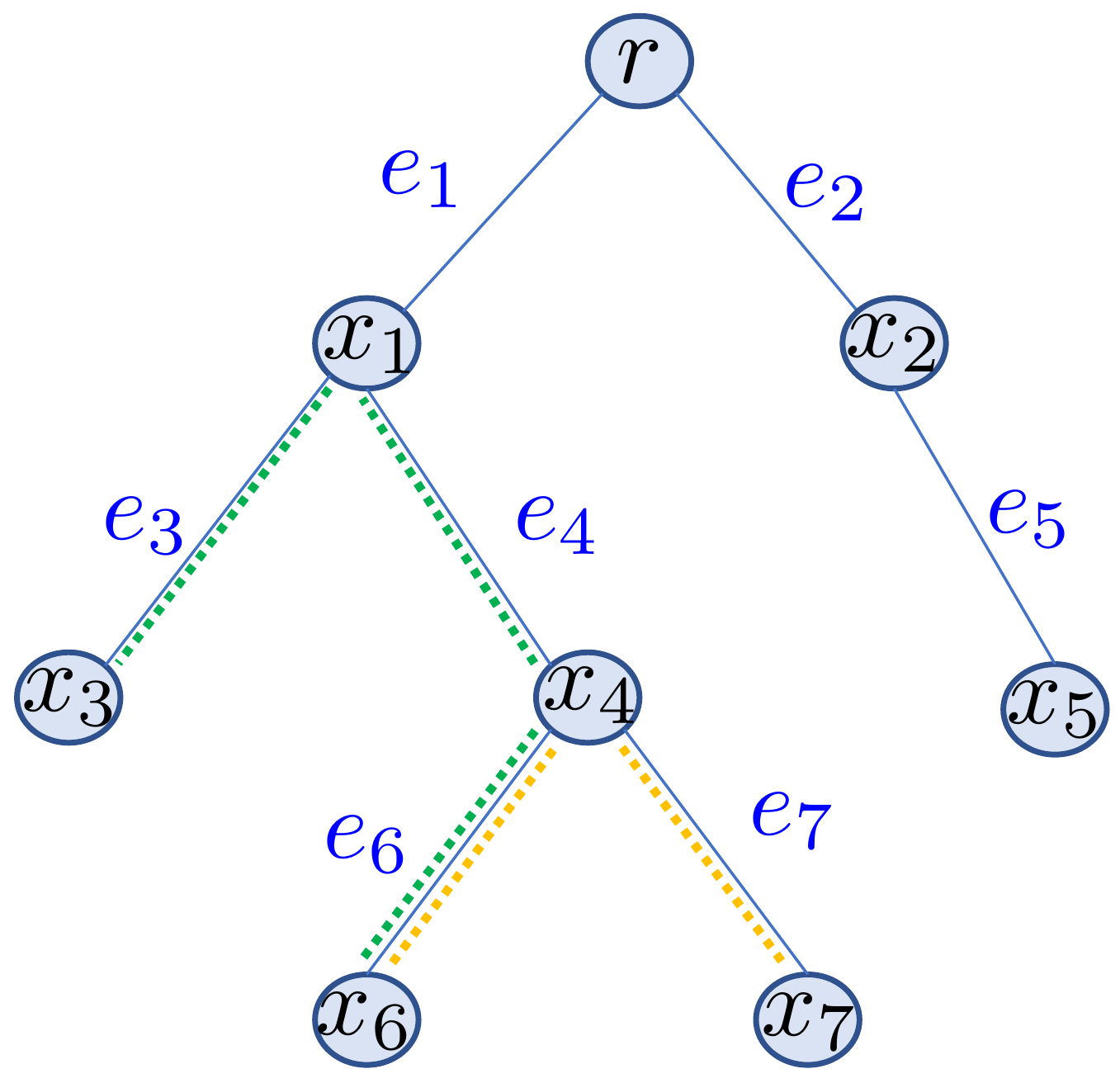}
  \end{center}
  \vspace{-6pt}
  \caption{An illustration for a tree with root $r$ where $x_1, x_2$ are at depth level $1$, and $x_6, x_7$ are at depth level $3$. Path $\Pp(x_3, x_6)$ contains $e_3, e_4, e_6$ (the green-dot path), $\Gamma(x_4) \!=\! \left\{x_4, x_6, x_7 \right\}$ (the yellow-dot subtree), $v_{e_4} = x_4$, and $u_{e_4} = x_1$.}
  \label{fg:Tree}
 \vspace{-10pt}
\end{figure}

\section{Tree-Wasserstein Distances: Optimal Transport with Tree Metrics}\label{sec:TWD}
Lozupone and co-authors \cite{lozupone2005unifrac, lozupone2007quantitative} first noticed, when proposing the \texttt{UniFrac} method in the metagenomics community, that the Wasserstein distance between two measures supported on the nodes of the same tree admits a closed form when the ground metric between the supports of the two measures is a tree metric. That method was used to compare microbial communities by measuring the phylogenetic distance between sets of taxa in a phylogenetic tree as the fraction of the branch length of the tree that leads to descendants from either one environment or the other, but not both \cite{lozupone2005unifrac}. In this section, we follow \cite{do2011sublinear, evans2012phylogenetic, mcgregor2013sketching} to leverage the geometric structure of tree metrics, and recall their main result.


Let $\Tt$ be a tree rooted at $r$ with non-negative edge lengths, and $d_{\TM}$ be the tree metric on $\Tt$. For nodes $x, z \in \Tt$, let $\Pp(x, z)$ be the (unique) path between $x$ and $z$ in $\Tt$, $\lambda$ is the unique Borel measure (i.e. length measure) on $\Tt$ such that $d_{\TM}(x, z) = \lambda(\Pp(x, z))$. We also write $\Gamma(x)$ for a set of nodes in the subtree of $\Tt$ rooted at $x$, defined as $\Gamma(x) = \left\{z \in \Tt \mid x \in \Pp(r, z) \right\}$. For each edge $e$ in $\Tt$, let $v_e$ be the deeper level node of edge $e$ (farther to the root), $u_e$ is the other node, and $w_e = d_{\TM}(u_e, v_e)$ is the non-negative length of that edge, illustrated in Figure~\ref{fg:Tree}. Then, $\TW$ not only has a closed form, but is negative definite.

\begin{proposition}\label{thr:TW_LT}
Given two measures $\mu$, $\nu$ supported on $\Tt$, and setting the ground metric to be $d_{\TM}$, then
\begin{eqnarray}\label{equ:OT_LT}
W_{d_{\TM}}(\mu, \nu) = \sum_{e \in \Tt} w_e \left| \mu(\Gamma(v_e)) - \nu(\Gamma(v_e)) \right|.
\end{eqnarray}
\end{proposition}


\begin{proof}
Following \cite{evans2012phylogenetic}, for any $f \in \Ff_{d_{\mathcal{T}}}$ such that $f(r) = 0$, there is an $\lambda$-a.e. unique Borel function $\texttt{f} : \Tt \rightarrow \left[-1, 1 \right]$ such that $f(x) = \int_{\Pp(r, x)} \texttt{f}(z) \lambda(dz) = \int_{\Tt} \mathbbm{1}_{z \in \Pp(r, x)} \texttt{f}(z) \lambda(dz)$. Intuitively, $f(x)$ models a flow along the (unique) path of the root $r$ and node $x$ where $\texttt{f}(z)$ controls a probability amount, received or provided by $f(x)$ on $dz$. Note that $\mathbbm{1}_{z \in \Pp(r, x)} = \mathbbm{1}_{x \in \Gamma(z)}$, then we have:
\[
\int_{\Tt} f(x)\mu(dx) = \int_{\Tt} \int_{\Tt} \mathbbm{1}_{z \in \Pp(r, x)} \texttt{f}(z) \lambda(dz) \mu(dx) = \int_{\Tt} \texttt{f}(z) \lambda(dz) \mu(\Gamma(z)). 
\]
Then, plugging this identity in Equation~(\ref{equ:KantorovichDualityProb}), we have:
\[
W_{d_{\TM}}(\mu, \nu) = \sup\left\{\int_{\Tt} \left(\mu(\Gamma(z)) - \nu(\Gamma(z)) \right) \texttt{f}(z) \lambda(dz) \right\} = \int_{\Tt} \left|\mu(\Gamma(z)) - \nu(\Gamma(z)) \right| \lambda(dz),
\]
since the optimal function $f^*$ corresponds to $\texttt{f}(z) = 1$ if $\mu(\Gamma(z)) \ge \nu(\Gamma(z))$, otherwise $\texttt{f}(z) = -1$. Moreover, we have $\mu(\Gamma(r)) = \nu(\Gamma(r)) = 1$, and $\lambda(\Pp(u_e, v_e)) = d_{\TM}(u_e, v_e) = w_e$. Therefore,
\[
W_{d_{\TM}}(\mu, \nu) = \sum_{e \in \Tt} w_e \left| \mu(\Gamma(v_e)) - \nu(\Gamma(v_e)) \right|,
\]
since the total mass flowing through edge $e$ is equal to the total mass in subtree $\Gamma(v_e)$.
\end{proof}

\begin{proposition}\label{thr:TW_ND}
The tree-Wasserstein distance $\TWA$ is negative definite.
\end{proposition}

\begin{proof}
Let $m$ be the number of edges in tree $\Tt$. From Equation~(\ref{equ:OT_LT}), $\mu(\Gamma(v_e))$ with $e \in \Tt$ can be considered as a feature map for probability distribution $\mu$ onto $\RR_{+}^{m}$. Consequently, the $\TW$ distance is equivalent to a weighted $\ell_1$ distance between these representations, with non-negative weights $w_e$, between these feature maps. Therefore, the tree-Wasserstein distance is negative definite\footnote{We follow here \cite[p. 66--67]{berg1984harmonic}, to define negative-definiteness, see review about kernels in the appendix.}.
\end{proof}

\section{Tree-Sliced Wasserstein by Sampling Tree Metrics}\label{sec:TSW}

Much as in sliced-Wasserstein (\SW) distances, computing \TW~distances requires choosing or sampling tree metrics. Unlike \SW~distances however, the space of possible tree metrics is far too large in practical cases to expect that purely random trees can lead to meaningful results. We consider in this section two adaptive methods to define tree metrics based on spatial information in both low and high-dimensional cases, using partitioning or clustering. We further average the \TW~distances corresponding to these ground tree metrics. This has the benefit of reducing quantization effects, or cluster sensitivity problems in which data points may be partitioned or clustered to adjacent but different hypercubes~\cite{indyk2003fast} or clusters respectively. We then define the tree-sliced Wasserstein kernel, that is the direct generalization of those considered by \cite{carriere17asliced, kolouri2016sliced}.

\begin{definition}\label{def:TSW}
For two measures $\mu,\nu$ supported on a set in which tree metrics $\left\{d_{{\TM}_i} \mid 1 \le i \le n \right\}$ can be defined, the tree-sliced-Wasserstein (\TSW) distance is defined as:
\begin{equation}\label{equ:TSW}
\TSW(\mu, \nu) = \frac{1}{n} \sum_{i=1}^n W_{d_{\TM_i}}(\mu, \nu).
\end{equation}
\end{definition}

Note that averaging of negative definite functions is trivially negative definite. Thus, following Definition~\ref{def:TSW} and Proposition~\ref{thr:TW_ND}, the \TSW~distance is also negative definite. Positive definite kernels can be therefore derived following \cite[Theorem 3.2.2, p.74]{berg1984harmonic}, and given $t>0$, $\mu, \nu$ on tree $\Tt$, we define the following tree-sliced-Wasserstein kernel,
\begin{equation}\label{equ:kTSW}
k_{\TSW}\!\left(\mu, \nu\right) = \exp\!\left(-t\, \TSW\!\left(\mu, \nu\right)\right).
\end{equation}
Very much like the Gaussian kernel, one can tune if needed the bandwidth parameter $t$ according to the learning task that is targeted, using e.g. cross validation. 



\paragraph{Adaptive methods to define tree metrics for the space of support data.} We consider sampling mechanisms to select tree metrics to be used in Definition~\ref{def:TSW}. 

One possibility is to sample tree metrics following the general idea that these tree metrics should approximate the original distance \cite{bartal1996probabilistic, bartal1998approximating, charikar1998approximating, fakcharoenphol2004tight, indyk2001algorithmic}. This was the original motivation for previous work focusing on approximating the \OT~distance with the Euclidean ground metric (a.k.a. $W_2$ metric) into $\ell_1$ metric for fast nearest neighbor search \cite{dong2019scalable, indyk2003fast}. Our goal is rather to sample tree metrics for the space of supports, and use those random tree metrics as ground metrics.  Much like $1$-dimensional projections do not offer interesting properties from a distortion perspective but remain useful for sliced-Wasserstein (\SW) distance, we believe that trees with large distortions can be useful. This follows the recent realization that solving exactly the \OT~problem leads to overfitting~\cite[\S8.4]{PeyreCuturiBook}, and therefore excessive efforts to approximate the ground metric using trees would be self-defeating since it would lead to overfitting within the computation of the Wasserstein metric itself.


\subparagraph{$\bullet$ Partition-based tree metrics.} For low-dimensional spaces of supports, one can construct a partition-based tree metric with a tree structure $\Tt$ as follows: 

\begin{algorithm}[H] 
\caption{\texttt{Partition\_HC}($\texttt{s}, X, \tilde{x}_\texttt{s}, \texttt{h}, H_{\Tt}$)} 
\label{alg:RP_TM} 
\begin{algorithmic}[1] 
    \REQUIRE \texttt{s}: side-$\ell$ hypercube, $X$: set of $m$ data points of $\RR^{\texttt{d}}$ in $\texttt{s}$, $\tilde{x}_\texttt{s}$: parent node, $\texttt{h}$: a current depth level, and $H_{\Tt}$: the predefined deepest level of tree $\Tt$.
     \IF{$m > 0$}
     	\IF{$\texttt{h} > 0$}	
		\STATE Node $\tilde{x}_{c}$ $\leftarrow$ a point center of $\texttt{s}$.
		\STATE Length of edge $(\tilde{x}_{\texttt{s}}, \tilde{x}_{c})$ $\leftarrow$ distance $(\tilde{x}_{\texttt{s}}, \tilde{x}_{c})$.
	\ELSE
		\STATE Node $\tilde{x}_{c} \leftarrow \tilde{x}_\texttt{s}$.
	\ENDIF
	\IF{$m > 1$ and $\texttt{h} < H_{\Tt}$}
     		\STATE Partition $\texttt{s}$ into $2^{\texttt{d}}$ side-$(\ell/2)$ hypercubes.
		\FOR{\textbf{each} side-$(\ell/2)$ hypercube $\texttt{s}_c$}
			\STATE $\tilde{X}_c \leftarrow$ data points of $X$ in $\texttt{s}_c$. 
			\STATE \texttt{Partition\_HC}($\texttt{s}_c, \tilde{X}_c, \tilde{x}_{c}, \texttt{h} + 1, H_{\Tt}$).
		\ENDFOR
	\ENDIF
     \ENDIF	
\end{algorithmic}
\end{algorithm}

Assume that data points are in a side-$(\beta/2)$ hypercube of $\RR^{\texttt{d}}$. We then randomly expand it into a hypercube $\texttt{s}_0$ with side at most $\beta$. Inspired by a series of grids in \cite{indyk2003fast}, we set the center of $\texttt{s}_0$ as the root of $\Tt$, and use a following recursive procedure to partition $\texttt{s}_0$. For each side-$\ell$ hypercube $\texttt{s}$, there are $3$ partitioning cases: (i) if $\texttt{s}$ does not contain any data points, we discard it, (ii) if $\texttt{s}$ contains $1$ data point, we use the center of $\texttt{s}$ (or the data point) as a node in $\Tt$, and (iii) if $\texttt{s}$ contains more than $1$ data point, we represent $\texttt{s}$ by its center as a node $x$ in $\Tt$, and equally partition $\texttt{s}$ into $2^{\texttt{d}}$ side-$(\ell/2)$ hypercubes for potential child nodes of $x$. We then apply the recursive partition procedure for those child hypercubes. One can use any metrics in $\RR^{\texttt{d}}$ to obtain lengths for edges in $\Tt$. Additionally, one can use a predefined deepest level of $\Tt$ as a stopping condition for the procedure. We summarize the recursive tree construction procedure in Algorithm~\ref{alg:RP_TM}. As desired, the random expansion of the original hypercube into $\texttt{s}_0$ creates a variety to partition data spaces. Note that Algorithm~\ref{alg:RP_TM} for constructing tree $\Tt$ is also known as the classical \texttt{Quadtree} algorithm \cite{samet1984quadtree} for 2-dimensional data (and later extended for high-dimensional data in \cite{backurs2019scalable, indyk2001algorithmic, indyk2017practical, indyk2003fast}).

\subparagraph{$\bullet$ Clustering-based tree metrics.} As in Algorithm~\ref{alg:RP_TM}, the number of partitioned hypercubes grows exponentially with respect to dimension ${\texttt{d}}$. To overcome this problem for high-dimensional spaces, we directly leverage the distribution of support data points to adaptively partition data spaces via clustering, inspired by the clustering-based approach for a space subdivision in Improved Fast Gauss Transform \cite{morariu2009automatic, yang2005efficient}. We derive a similar recursive procedure as in the partition-based tree metrics, but apply the farthest-point clustering \cite{gonzalez1985clustering} to partition support data points, and replace centers of hypercubes by cluster centroids as nodes in $\Tt$. In practice, we fix the same number of clusters $\kappa$ when performing the farthest-point clustering (replace the partition in line $9$ in Algorithm~\ref{alg:RP_TM}). $\kappa$ is typically chosen via cross-validation. In general, one can apply any favorite clustering methods. We use the farthest-point clustering due to its fast computation. In particular, the complexity of the farthest-point clustering into $\kappa$ clusters for $n$ data points is $O\!\left(n\log\kappa\right)$ using the algorithm in \cite{feder1988optimal}. Using different random initializations for the farthest-point clustering, we recover a simple sampling mechanism to obtain random tree metrics.

\section{Relations to Other Work}\label{sec:relations}

\paragraph{OT with ground ultrametrics.} An ultrametric is also known as non-Archimedean metric, or isosceles metric \cite{shkarin2004isometric}. Ultrametrics strengthen the triangle inequality to a strong inequality (i.e., for any $x, y, z$ in an ultrametric space, $d(x, z) \le \max\!\left(d(x,y), d(y, z) \right)$). Note that binary metrics are a special case of ultrametrics since binary metrics satisfy the strong inequality. Following \cite[\S1, p.245--247]{johnson1967hierarchical}, an ultrametric implies a tree structure which can be constructed by hierarchical clustering schemes. Therefore, an ultrametric is a tree metric. Furthermore, we note that ultrametrics have similar spirits with strong kernels and hierarchy-induced kernels which are key components to form valid optimal assignment kernels for graph classification applications \cite{kriege2016valid}.



\paragraph{Connection with OT with Euclidean ground metric $W_2(\cdot, \cdot)$.} Let $d_{\TM}^{H}$ be a partition-based tree metric where $H$ is the depth level of corresponding tree $\Tt$, at which all support data points are separated into different hypercubes (i.e., Algorithm~\ref{alg:RP_TM} stops at depth level $H$). Edges in $\Tt$ are computed by Euclidean distance. Let $\beta$ be the side of the randomly expanded hypercube. Given two $\texttt{d}$-dimensional point clouds $\tilde{\mu}, \tilde{\nu}$ with the same cardinality (i.e., discrete uniform measures), and denote $\TW$ with $d_{\TM}^{H}$ as $W_{d_{\TM}^H}$. Then, 
\[
W_2(\tilde{\mu}, \tilde{\nu}) \le W_{d_{\TM}^H}(\tilde{\mu}, \tilde{\nu})/2 + \beta \sqrt{\texttt{d}}/2^H.
\]
The proof is given in the appendix material. Moreover, we also investigate the empirical relation between the \TSW~distance and the $W_2$ distance in the appendix material, in which empirical results indicate that the \TSW~distance agrees more with $W_2$ as the number of tree-slices used to define the \TSW~distance is increased.

\paragraph{Connection with embedding $W_2$ metric into $\ell_1$ metric for fast nearest neighbor search.} As discussed earlier, our goal is neither to approximate \OT~distance using trees as in \cite{bartal1996probabilistic, bartal1998approximating, charikar1998approximating, fakcharoenphol2004tight, indyk2001algorithmic}, nor to embed $W_2$ metric into $\ell_1$ metric as in \cite{dong2019scalable, indyk2003fast}, but rather to sample tree metrics to define an extended variant of the sliced-Wasserstein distance.  When using the \texttt{Quadtree} algorithm (as in Algorithm~\ref{alg:RP_TM}) to sample tree metrics for the \TSW~distance, then the resulted \TSW~distance is in the same spirit as the embedding approach in \cite{indyk2003fast} where the authors embedded $W_2$ metric into $\ell_1$ metric by using a series of grids.


\paragraph{OT with tree metrics.} There are a few work related to our considered class of $\OT$ with tree metrics \cite{kloeckner2015geometric, sommerfeld2018inference}. In particular, Kloeckner \cite{kloeckner2015geometric} studied geometric properties of $\OT$ space for measures on an ultrametric space, and Sommerfeld and Munk \cite{sommerfeld2018inference} focused on statistical inference for empirical $\OT$ on finite spaces including tree metrics. 



\section{Experimental Results}\label{sec:experiments}

In this section, we evaluated the proposed \TSW~kernel $k_{\TSW}$ (Equation~\eqref{equ:kTSW}) for comparing empirical measures in word embedding-based document classification and topological data analysis.

\subsection{Word Embedding-based Document Classification}\label{sec:documentclassification}
Kusner et al. \cite{kusner2015word} proposed Word Mover's distances for document classification. Each document is regarded as an empirical measure where each word and its frequency are considered as a support and a corresponding weight respectively. Kusner et al. \cite{kusner2015word} used word embedding such as \textit{word2vec} to map each word to a vector data point. Equivalently, Word Mover's distances are $\OT$ metrics between empirical measures (i.e., documents) where its ground cost is a metric on the word embedding space.

\paragraph{Setup.} We evaluated $k_{\TSW}$ on four datasets: \texttt{TWITTER}, \texttt{RECIPE}, \texttt{CLASSIC} and \texttt{AMAZON}, following the approach of Word Mover's distances~\cite{kusner2015word}, for document classification with SVM. Statistical characteristics for those datasets are summarized in Figure~\ref{fg:Time}. We used the \textit{word2vec} word embedding \cite{mikolov2013distributed}, pre-trained on Google News\footnote{https://code.google.com/p/word2vec}, containing about $3$ million words/phrases. \textit{word2vec} maps these words/phrases into vectors in $\RR^{300}$. Following \cite{kusner2015word}, for all datasets, we removed all SMART stop word \cite{salton1988term}, and further dropped words in documents if they are not available in the pre-trained \textit{word2vec}. We used two baseline kernels in the form of $\exp(-td)$ where $d$ is a document distance and $t>0$, for two corresponding baseline document distances based on Word Mover's: (i) \OT~with Euclidean ground metric \cite{kusner2015word}, and (ii) sliced-Wasserstein, denoted as $k_{\OT}$ and $k_{\SW}$ respectively. For \TSW~distance in $k_{\TSW}$, we consider $n_s$ randomized clustering-based tree metrics, built with a predefined deepest level $H_{\Tt}$ of tree $\Tt$ as a stopping condition. We also regularized for kernel $k_{\OT}$ matrices due to its indefiniteness by adding a sufficiently large diagonal term as in \cite{cuturi2013sinkhorn}. For SVM, we randomly split each dataset into $70\%/30\%$ for training and test with $100$ repeats, choose hyper-parameters through cross validation, choose $1/t$ from $\left\{1, q_{10}, q_{20}, q_{50} \right\}$ where $q_s$ is the $s\%$ quantile of a subset of corresponding distances, observed on a training set, use one-vs-one strategy with Libsvm \cite{chang2011libsvm} for multi-class classification, and choose SVM regularization from $\left\{10^{-2:1:2} \right\}$. We ran experiments with Intel Xeon CPU E7-8891v3 (2.80GHz), and 256GB RAM.  


\subsection{Topological Data Analysis (TDA)}\label{sec:TDA}

\TDA~has recently gained interest within the machine learning community \cite{carriere17asliced,kusano2017kernelJMLR, le2018persistence,reininghaus2015stable}. \TDA~is a powerful tool for statistical analysis on geometric structured data such as linked twist maps, or material data. \TDA~employs algebraic topology methods, such as persistence homology, to extract robust topological features (i.e., connected components, rings, cavities) and output $2$-dimensional point multisets, known as persistence diagrams (\PD) \cite{edelsbrunner2008persistent}. Each $2$-dimensional point in \PD~summarizes a lifespan, corresponding to birth and death time as its coordinates, of a particular topological feature.


\paragraph{Setup.}  
We evaluated $k_{\TSW}$ for orbit recognition and object shape classification with support vector machines (SVM), as well as change point detection for material data analysis with kernel Fisher discriminant ratio (KFDR)~\cite{harchaoui2009kernel}. Generally, we followed the same setting as in \cite{le2018persistence} for these \TDA~experiments. We considered five baseline kernels for \PD: (i) persistence scale space ($k_{\PSS}$)~\cite{reininghaus2015stable}, (ii) persistence weighted Gaussian ($k_{\PWG}$)~\cite{kusano2017kernelJMLR}, (iii) sliced-Wasserstein ($k_{\SW}$)~\cite{carriere17asliced}, (iv) persistence Fisher ($k_{\PF}$) \cite{le2018persistence}, and (v) optimal transport\footnote{We used a fast \OT~implementation (e.g. on MPEG7 dataset, it took $7.98$ seconds while the popular mex-file with Rubner's implementation required $28.72$ seconds).}, defined as $k_{\OT} = \exp(-t d_{\OT})$ for $t>0$, and also further regularized its kernel matrices by adding a sufficiently large diagonal term due to its indefiniteness as in \S\ref{sec:documentclassification}. For \TSW~distance in $k_{\TSW}$, we considered $n_s$ randomized partition-based tree metrics, built with a predefined deepest level $H_{\Tt}$ of tree $\Tt$ as a stopping condition. 


Let $\Dg_i = (x_1, x_2, \dotsc, x_n)$ and $\Dg_j = (z_1, z_2, \dotsc, z_m)$ be two \PD~where $x_i \mid_{1 \le i \le n}, z_j \mid_{1 \le j \le m} \in \RR^2$, and $\Theta = \{(a, a) \mid a \in \RR\}$ be the diagonal set. Denote $\Dg_{i\Theta} = \{\Pi_{\Theta}(x) \mid x \in \Dg_i \}$ where $\Pi_{\Theta}(x)$ is a projection of $x$ on $\Theta$. As in $\SW$ distance between $\Dg_i$ and $\Dg_j$ \cite{carriere17asliced}, we use transportation plans between $(\Dg_i \cup \Dg_{j\Theta})$ and $(\Dg_j \cup \Dg_{i\Theta})$ for \TW~(in Equation~\eqref{equ:TSW} of \TSW) and \OT~distances. We typically used a cross validation to choose hyper-parameters, and followed corresponding authors of those baseline kernels to form sets of candidates. For $k_{\TSW}$ and $k_{\OT}$, we chose $1/t$ from $\left\{1, q_{10}, q_{20}, q_{50} \right\}$. Similar as in \S\ref{sec:documentclassification}, we used one-vs-one strategy with Libsvm for multi-class classification, $\left\{10^{-2:1:2} \right\}$ as a set of regularization candidates, and a random split $70\%/30\%$ for training and test with $100$ repeats for SVM, and DIPHA toolbox\footnote{https://github.com/DIPHA/dipha} to extract \PD. 


\paragraph{Orbit recognition.} Adams et al. \cite[\S 6.4.1]{adams2017persistence} proposed a synthesized dataset for link twist map, a discrete dynamical system to model flows in DNA microarrays \cite{hertzsch2007dna}. There are $5$ classes of orbits. As in \cite{le2018persistence}, we generated $1000$ orbits for each class where each orbit contains $1000$ points. We considered $1$-dimensional topological features for \PD, extracted with Vietoris-Rips complex filtration   \cite{edelsbrunner2008persistent}.

\paragraph{Object shape classification.} We evaluated object shape classification on a $10$-class subset of MPEG7 dataset \cite{latecki2000shape}, containing $20$ samples for each class as in \cite{le2018persistence}. For simplicity, we used the same procedure as in \cite{le2018persistence} to extract $1$-dimensional topological features for \PD~with Vietoris-Rips complex filtration\footnote{Turner et al. \cite{turner2014persistent} proposed a more complicated and advanced filtration for this task.} \cite{edelsbrunner2008persistent}. 


\paragraph{Change point detection for material data analysis.} We considered granular packing system~\cite{francois2013geometrical} and SiO$_2$~\cite{nakamura2015persistent} datasets for change point detection problem with KFDR as a statistical score. As in \cite{kusano2017kernelJMLR, le2018persistence}, we extracted $2$-dimensional topological features for \PD~in granular packing system dataset, $1$-dimensional topological features for \PD~ in SiO$_2$ dataset, both with ball model filtration, and set $10^{-3}$ for the regularization parameter in KFDR. KFDR graphs for these datasets are shown in Figure \ref{fg:KFDR}. For granular tracking system dataset, all kernel approaches obtain the change point as the $23^{rd}$ index, which support an observation result (corresponding id = $23$) in \cite{anonymous72} . For SiO$_2$ dataset, results of all kernel methods are within a supported range ($35 \le$ id $\le 50$), obtained by a traditional physical approach \cite{elliott1983physics}. The KFDR results of $k_{\TSW}$ compare favorably with those of other baseline kernels. As shown in Figure~\ref{fg:Time}, $k_{\TSW}$ is faster than other baseline kernels. We note that we omit the baseline kernel $k_{\OT}$ for this application since computation of \OT~distance is out of memory.

\subsection{Results of SVM, Time Consumption and Discussion} 


The results of SVM and time consumption for kernel matrices in \TDA, and word embedding based document classification are illustrated in Figure~\ref{fg:ACC} and Figure~\ref{fg:Time} respectively. The performances of $k_{\TSW}$ compare favorably with other baseline kernels. Moreover, the computational time of $k_{\TSW}$ is much less than that of $k_{\OT}$. Especially, in \texttt{CLASSIC} dataset, it took less than $3$ hours for $k_{\TSW}$ while more than $8$ days for $k_{\OT}$. Note that $k_{\TSW}$ and $k_{\SW}$ are positive definite while $k_{\OT}$ is not. The indefiniteness of $k_{\OT}$ may affect its performances in some applications, e.g. $k_{\OT}$ performs worse in TDA applications, but works well for documents with word embedding applications. The fact that $\SW$ only considers $1$-dimensional projections may limit its ability to capture high-dimensional structure in data distributions \cite{csimcsekli2018sliced}. \TSW~distance remedies this problem by using clustering-based tree metrics which directly leverage distributions of support data points. Furthermore, we also illustrate a trade-off of performances and computational time for different parameters in tree-sliced-Wasserstein distances for $k_{\TSW}$ on \texttt{TWITTER} dataset in Figure~\ref{fg:TWITTERKTW}. For tree-sliced-Wasserstein $\TSW$ for $k_{\TSW}$, performances are usually improved with more slices ($n_s$), but they come with a trade-off of more computational time. In these applications, we observed that a good trade-off for $n_s$ of tree-sliced-Wasserstein is about $10$ slices. Many further results can be seen in the appendix.

\begin{figure}
    \centering
    \begin{subfigure}[t]{0.77\textwidth}
        \centering
        \includegraphics[width=\linewidth]{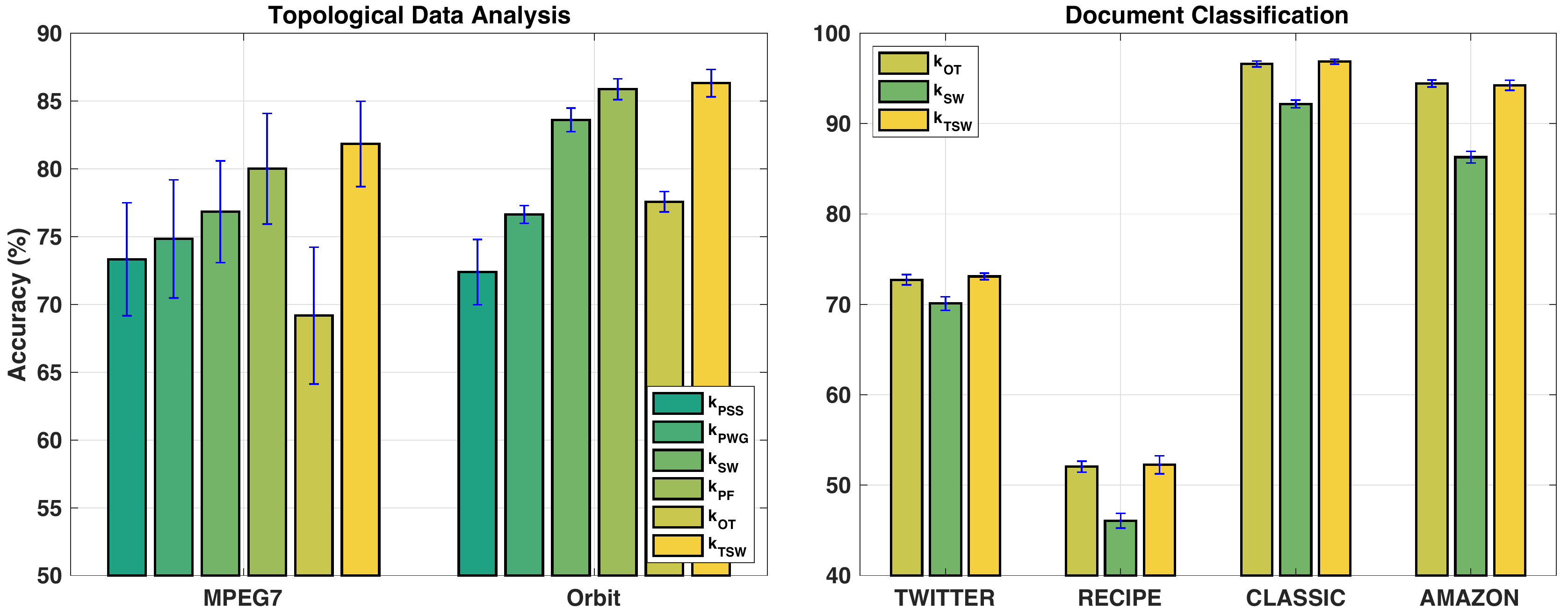} 
        \caption{SVM results for $\TDA$ and document classification.} \label{fg:ACC}
    \end{subfigure}
    
     \begin{subfigure}[t]{0.52\textwidth}
    \centering
        \includegraphics[width=\linewidth]{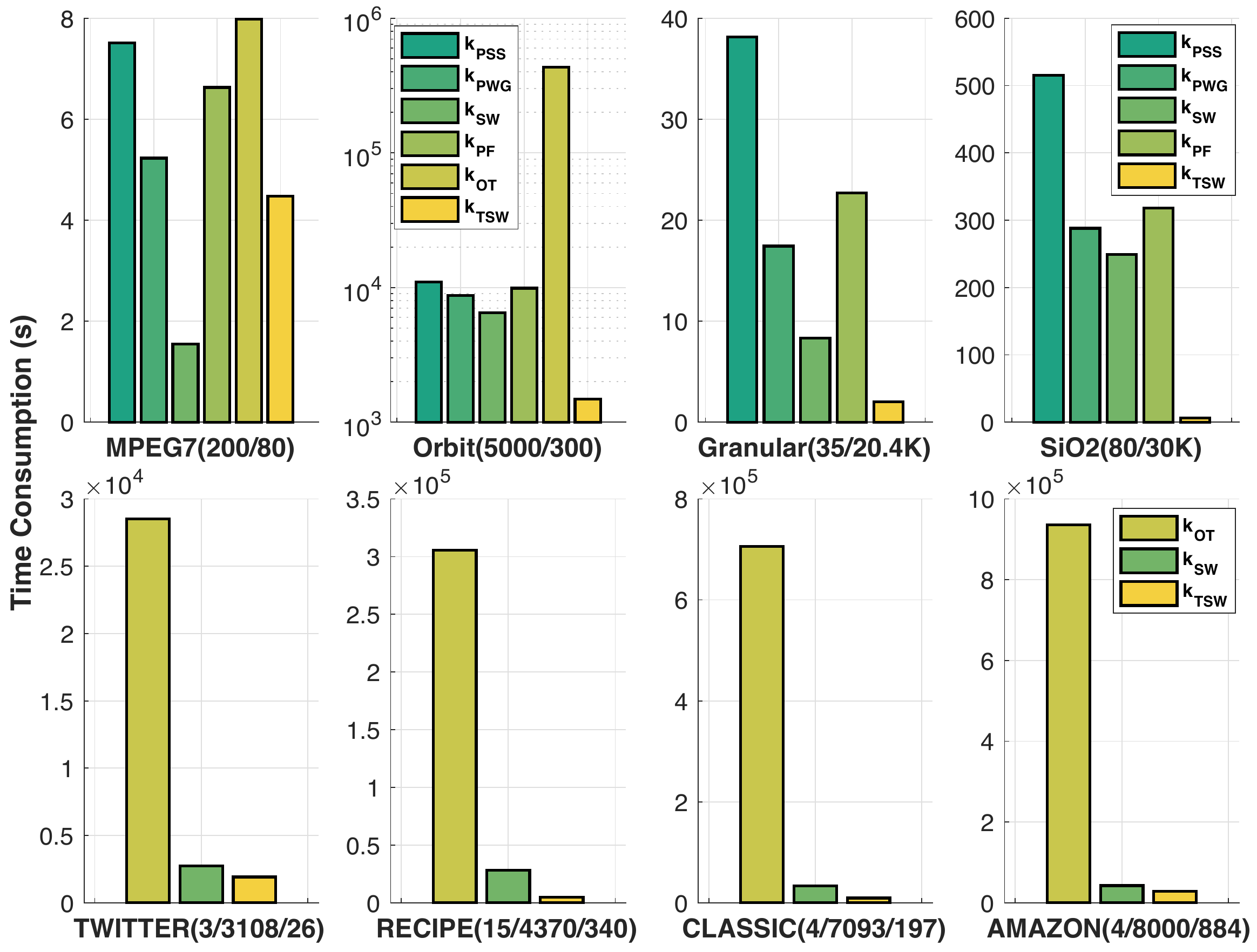} 
        \caption{Corresponding time consumption of kernel matrices for $\TDA$ and document classification.} \label{fg:Time}
    \end{subfigure}
    \hfill
    \begin{subfigure}[t]{0.46\textwidth}
        \centering
        \includegraphics[width=\linewidth]{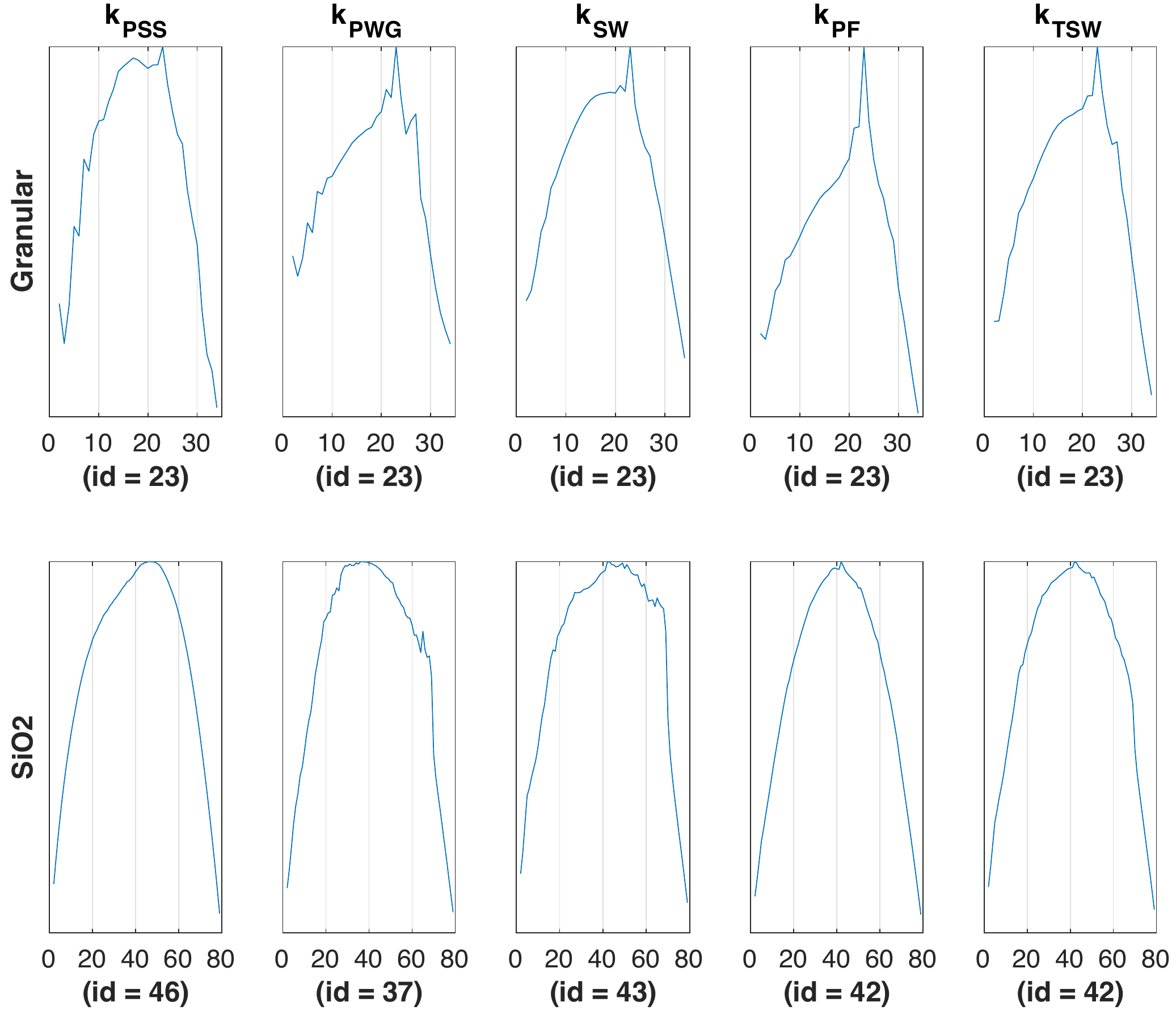} 
        \caption{The KFDR graphs on granular packing system and SiO$_2$ datasets.} \label{fg:KFDR}
    \end{subfigure}
    
    \caption{Experimental results. In (\subref{fg:ACC}), for a trade-off between time consumption and performances, results of $\TDA$ are reported for $k_{\TSW}$ with $(n_s\!=\!6, H_{\Tt}\!=\!6)$, and $(n_s\!=\!12, H_{\Tt}\!=\!5)$ in MPEG7 and Orbit datasets respectively. For document classification, results are reported for $k_{\SW}$ with $(n_s\!=\!20)$, and for $k_{\TSW}$ with $(n_s\!=\!10, H_{\Tt}\!=\!6, \kappa\!=\!4)$. In (\subref{fg:Time}), the numbers in the parenthesis: for $\TDA$ (the first row), are the number of $\PD$ and the maximum number of points in $\PD$ respectively; for document classification (the second row), are the number of classes, the number of documents, and the maximum number of unique words for each document respectively. In (\subref{fg:KFDR}), for $k_{\TSW}$, \TSW~distances are computed with $(n_s\!=\!12, H_{\Tt}\!=\!6)$.}
\end{figure}

    \begin{figure}
    \centering
        \includegraphics[width=\linewidth]{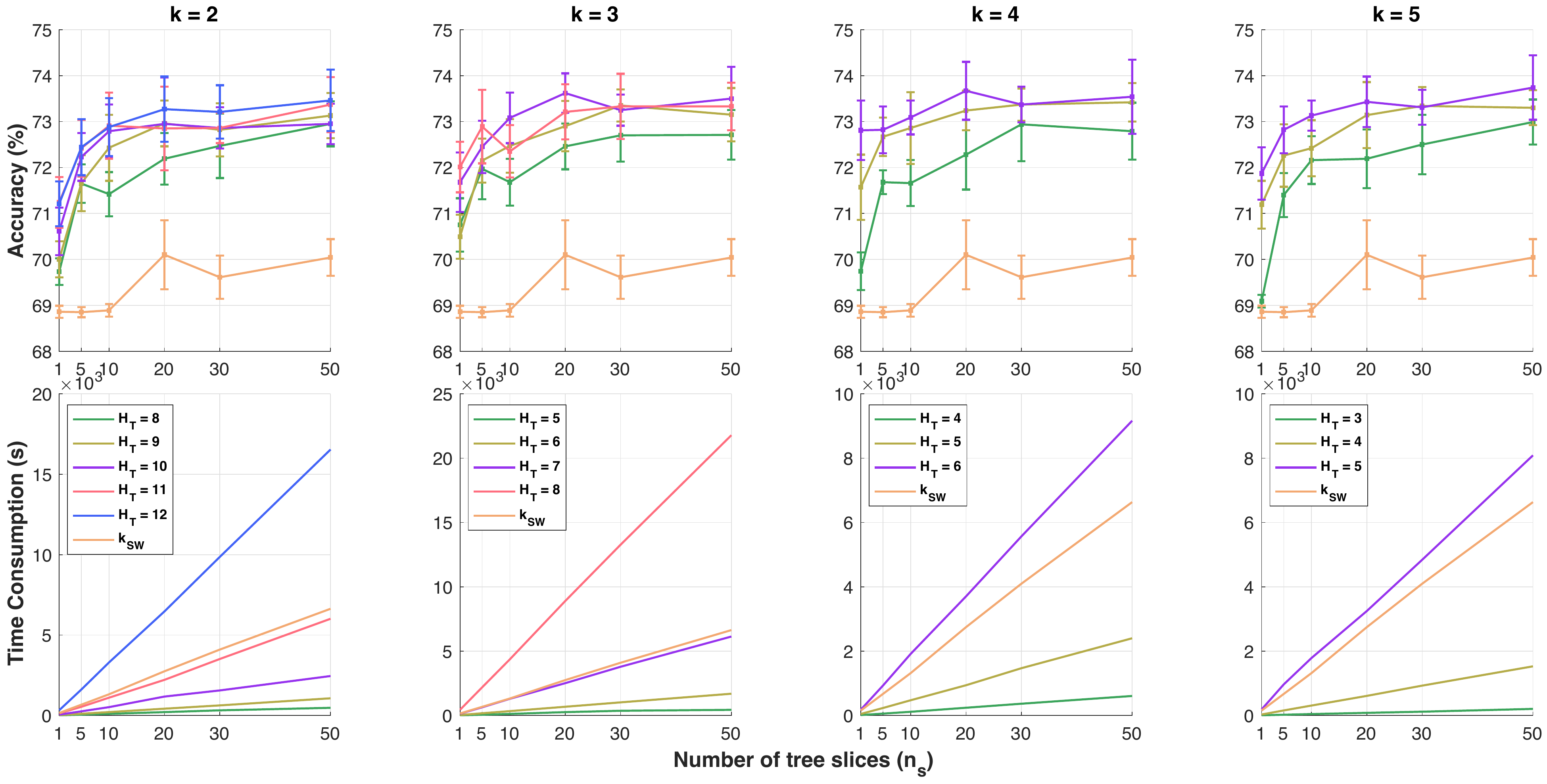} 
        \caption{SVM results and time consumption of kernel matrices of $k_{\TSW}$ with different $(n_s, H_{\Tt}, \kappa)$, and $k_{\SW}$ with different $n_s$ on \texttt{TWITTER} dataset.}
         \label{fg:TWITTERKTW}
    \end{figure}
    
\section{Conclusion}\label{sec:conclusion}

In this work, we proposed positive definite tree-(sliced)-Wasserstein kernel on $\OT$ geometry by considering a particular class of ground metrics, namely tree metrics. Much like the univariate Wasserstein distance, the tree-(sliced)-Wasserstein distance has a closed form, and is also negative definite. We also provide two sampling schemes to generate tree metrics for both high-dimensional and low-dimensional spaces. Leveraging random tree-metrics, we have proposed a new generalization of sliced-Wasserstein metrics that has more flexibility and degrees of freedom, by choosing a tree rather than a line, especially in high-dimensional spaces. The questions of sampling efficiently tree metrics from data points for tree-sliced-Wasserstein distance, as well as using them for more involved parametric inference are left for future work.

\subsubsection*{Acknowledgments}
We thank anonymous reviewers for their comments. TL acknowledges the support of JSPS KAKENHI Grant number 17K12745. MY was supported by the JST PRESTO program JPMJPR165A.

\appendix

\section{Detailed Proofs}
In this section, we give detailed proofs for the inequality in the connection with \OT~with Euclidean ground metric (i.e. $W_2$ metric) for \TW~distance, and investigate an empirical relation between \TSW~and $W_2$ metric, especially when one increases the number of tree-slices in \TSW. Additionally, we also provide proofs for negative definiteness of $\ell_1$ distance (used in the proof of Proposition 2 in the main text), and indefinite divisibility for \TSW~kernel.

\subsection{Proof of: $W_2(\tilde{\mu}, \tilde{\nu}) \le W_{d_{\TM}^H}(\tilde{\mu}, \tilde{\nu})/2 + \beta \sqrt{\texttt{d}}/2^H$}\label{sec:connectionWL2} 

For two point clouds $\tilde{\mu}, \tilde{\nu}$ containing $n$ data points $x_i \mid_{1 \le i \le n}, z_j \mid_{1 \le j \le n}$ respectively, let $c$ be a ground cost metric, and $\Sigma_n$ be the set of all permutations of $n$ elements, the \OT~can be reformulated as an optimal assignment problem as follow:
\begin{equation}\label{equ:OptimalAssignment}
W_{c}(\tilde{\mu}, \tilde{\nu}) = \min_{\sigma \in \Sigma_n} \frac{1}{n}\sum_{i=1}^n c(x_i, z_{\sigma(i)}).
\end{equation}

At a height level $i$ in $\Tt$, the maximum Euclidean distance between any two data points in a same hypercube, denoted as $\Delta_i$, we have 
\[
\Delta_i = \beta \sqrt{\texttt{d}} / 2^i. 
\]

Let $E_{i(i+1)}$ be a set of all edges between a height level $i$ and a height level $(i+1)$ in $\Tt$. So, for any $e \in E_{i(i+1)}$, we have 
\[
w_e = \beta \sqrt{\texttt{d}} / 2^{i+1}.
\] 

Let $q_i$ be the number of matched pairs at a height level $i$. Consequently, $(n - q_i)$ is the number of unmatched pairs at the height level $i$. Moreover, for the number of unmatched pairs at the height level $i$, observe that
\begin{equation}\label{equ:UM}
n - q_i = \frac{1}{2}\sum_{e \in E_{(i-1)i}} {\left| \tilde{\mu}(\Gamma(v_e)) - \tilde{\nu}(\Gamma(v_e)) \right|}.
\end{equation}
In the right hand side of Equation~(\ref{equ:UM}), for all edges between the height level $(i-1)$ and $i$ of tree $\Tt$. Note that the total mass in subtree $\Gamma(v_e)$ is equal to the total mass flowing through edge $e$, we have $\Pp(r, x) \mid_{x \in \tilde{\mu}}$ and $\Pp(r, z) \mid_{z \in \tilde{\nu}}$ count the total number of visits on each edge in $E_{i(i+1)}$ for $\tilde{\mu}$ and $\tilde{\nu}$ respectively, and their absolute different number is twice to the number of unmatched pairs at the height level $i$, as described in Equation~(\ref{equ:UM}).

Therefore, the \TW~distance is as follow:
\begin{equation}\label{equ:dTW}
W_{d_{\TM}^H}(\tilde{\mu}, \tilde{\nu}) = \frac{1}{n}\!\sum_{i=0}^{H-1}{2w_{\bar{e}_{i(i+1)}} (n - q_{i+1})} = \frac{1}{n}\!\sum_{i=0}^{H-1} {\Delta_i(n - q_{i+1})},
\end{equation}
where $\bar{e}_{i(i+1)}$ is an edge in $E_{i(i+1)}$, and note that $\Delta_i = 2w_{\bar{e}_{i(i+1)}}$.

Moreover, we have $\left(q_i - q_{i+1} \right)$ is the number of pairs matched at a height level $i$, but unmatched at a height level $(i+1)$. Additionally, note that $q_0 = n$, $q_H=0$, and $\Delta_{i} = \Delta_{i-1}/2$, then we have
\begin{eqnarray}\label{eq:BoundOT}
W_2(\tilde{\mu}, \tilde{\nu}) \le \frac{1}{n}\!\sum_{i=0}^{H-1} {\Delta_i \left(q_i - q_{i+1} \right)} \qquad \qquad \qquad \qquad \qquad \quad \,\,\,\,  \\
= W_{d_{\TM}^H}(\tilde{\mu}, \tilde{\nu}) - \frac{1}{n}\!\sum_{i=0}^{H-1}{\Delta_i(n-q_i)} \qquad \qquad \qquad  \\
= W_{d_{\TM}^H}(\tilde{\mu}, \tilde{\nu}) - \frac{1}{n}\!\sum_{i=1}^{H}{\Delta_{i-1}(n-q_i)/2 + \Delta_H} \quad \,\,\,  \\
= W_{d_{\TM}^H}(\tilde{\mu}, \tilde{\nu})/2 + \beta \sqrt{\texttt{d}}/2^H. \qquad \qquad \qquad \qquad \,\, 
\end{eqnarray}
For the first equal, we added and subtracted $n$ for the term in the parenthesis and note Equation~(\ref{equ:dTW}) for $W_{d_{\TM}^H}$. For the second equal, in the second term, for the element with $i=0$, note that $(n - q_0) = 0$, we added and subtracted the element with $i=H$. For the third equal, we grouped the first two terms and note that $\Delta_H = \beta \sqrt{\texttt{d}} / 2^H$ for the third term.

\begin{figure}
  \begin{center}
    \includegraphics[width=0.7\textwidth]{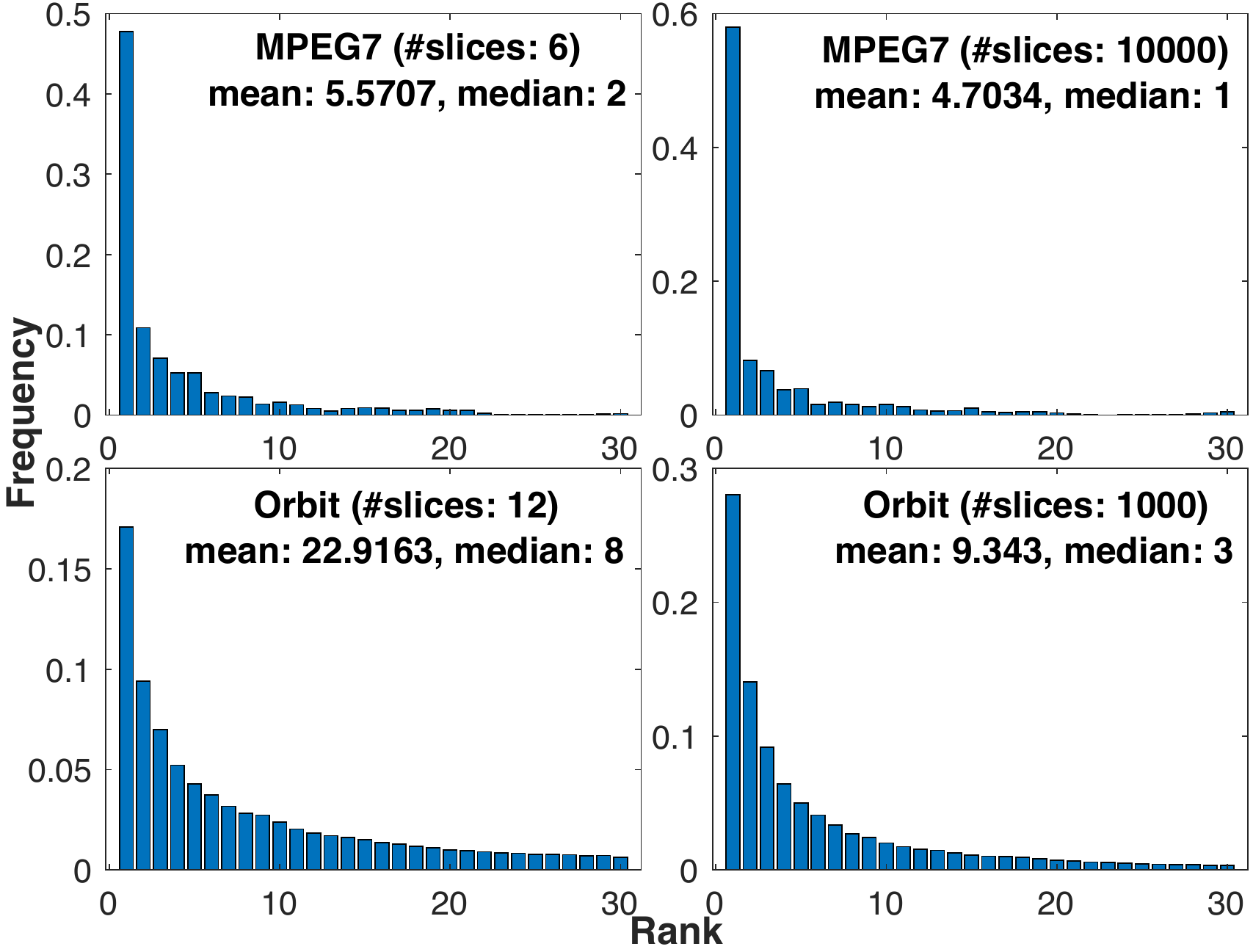}
  \end{center}
  \vspace{-6pt}
  \caption{A frequency of near-neighbor rank on the $W_2$ space for the nearest neighbor w.r.t. TSW.}
  \label{fg:RankOrderTSW}
\end{figure}

\paragraph{Empirical relation between TSW and OT with Euclidean ground metric.} The hypercube tree-sliced metric (i.e. partition-based tree metric) is our suggestion to build practical tree metrics for \TSW~when used on low-dimensional data spaces. We emphasize that we do not try to mimic the Euclidean \OT~(i.e. $W_2$) or the sliced-Wasserstein (\SW), but rather propose a variant of \OT~distance. As stated in the main text, \SW~is a special case of \TSW. From an empirical point of view, we have carried out the following experiment to investigate an empirical relation between \TSW~and $W_2$ distance: 

For a query point $q$, let $p$ be its nearest neighbor w.r.t. TSW. Figure~\ref{fg:RankOrderTSW} illustrates that $p$ is very likely among the top $5$ on MPEG7 dataset, and top $10$ on Orbit dataset, near neighbors on the $W_2$ space. Results are averaged over $1000$ runs of random split $90\%/10\%$ for training and test. When the number of tree-slices in TSW increases, the $W_2$ near-neighbor rank of $p$ is improved. These empirical results suggest that TSW may agree with some aspects of $W_2$.

\subsection{Proof of: Negative Definiteness for $\ell_1$ Distance}

For two real numbers $a, b$, the function $(a, b) \mapsto (a-b)^2$ is obviously negative definite. Following \cite[Corollary 2.10, p.78]{berg1984harmonic}, the function $(a, b) \mapsto \left|a-b\right|$ is negative definite. Therefore, $\ell_1$ distance is a sum of negative definite functions. Thus, $\ell_1$ distance is negative definite.

\subsection{Indefinite Divisibility for Tree-Sliced-Wasserstein Kernel}

Inspired by Le and Yamada \cite{le2018persistence}, we derive the following proof of indefinite divisibility for the \TSW~kernel. For probability $\mu, \nu$ on tree $\Tt$, and $i \in \NN^*$, let $k_{\TSW_{i}}(\mu, \nu) = \exp(-\frac{t}{i} \TSW(\mu, \nu))$. We have $k_{\TSW}(\cdot, \cdot) = \left(k_{\TSW_i}(\cdot, \cdot)\right)^i$, and $k_{\TSW_i}(\cdot, \cdot)$ is positive definite. Following \cite[\S3, Definition 2.6, p.76]{berg1984harmonic}, $k_{\TSW}$ is indefinitely divisible. Therefore, one does not need to recompute the Gram matrix of \TSW~kernel for each choice of $t$, since it indeed suffices to compute \TSW~distances between empirical measures in a training set once.

\section{More Experimental Results}\label{sec:MoreExpResults}

We provide many further experimental results for our proposed tree-Wasserstein kernel on topological data analysis ($\TDA$) and word embedding-based document classification.

\subsection{Topological Data Analysis}

\paragraph{Orbit recognition.} Figure \ref{fg:OrbitKTW} shows SVM results and time consumption of kernel matrices of $k_{\TW}$ with different $(n_s, H_{\Tt})$ on Orbit dataset.

\paragraph{Object shape classification.} Figure \ref{fg:MPEG7KTW} shows SVM results and time consumption of kernel matrices of $k_{\TW}$ with different $(n_s, H_{\Tt})$ on MPEG7 dataset.

\paragraph{Change point detection for material data analysis.} Figure \ref{fg:GranularKTW} and Figure \ref{fg:SiO2KTW} show time consumption of kernel matrices of $k_{\TW}$ with different $(n_s, H_{\Tt})$ on granular packing system and SiO$_2$ datasets respectively.

\subsection{Word Embedding-based Document Classification}


Figure~\ref{fg:RECIPE_ALL_KTW}, Figure~\ref{fg:CLASSIC_ALL_KTW}, and Figure~\ref{fg:AMAZON_ALL_KTW} show SVM results and time consumption of kernel matrices of $k_{\TW}$ with different $(n_s, H_{\Tt}, \kappa)$, and $k_{\SW}$ with different $n_s$ on \texttt{RECIPE}, \texttt{CLASSIC}, and \texttt{AMAZON} datasets respectively.

\begin{figure}
  \begin{center}
    \includegraphics[width=0.5\textwidth]{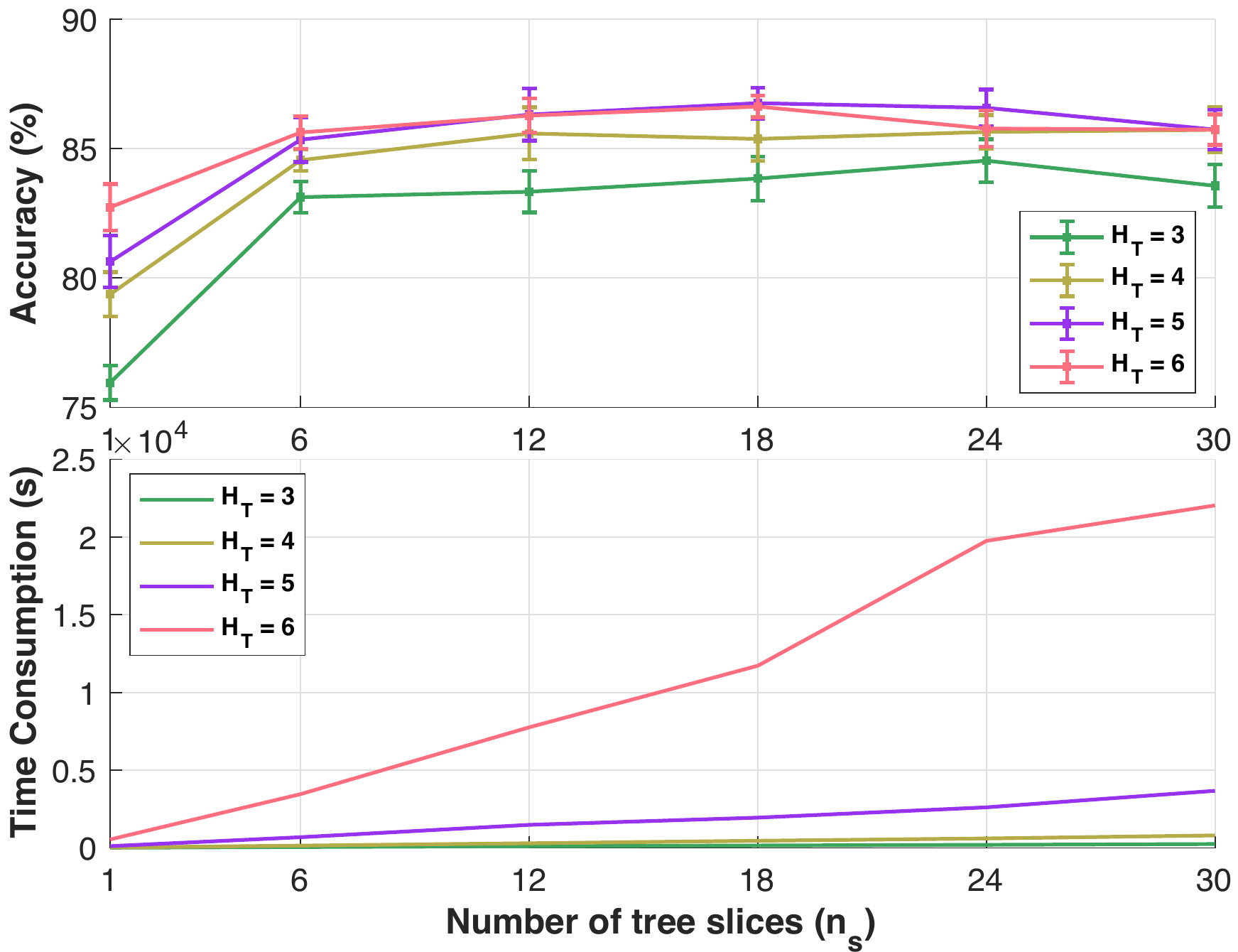}
  \end{center}
  \caption{Results of SVM and time consumption of kernel matrices of $k_{\TW}$ with different $(n_s, H_{\Tt})$ on Orbit dataset.}
  \label{fg:OrbitKTW}
\end{figure}

\begin{figure}
  \begin{center}
    \includegraphics[width=0.5\textwidth]{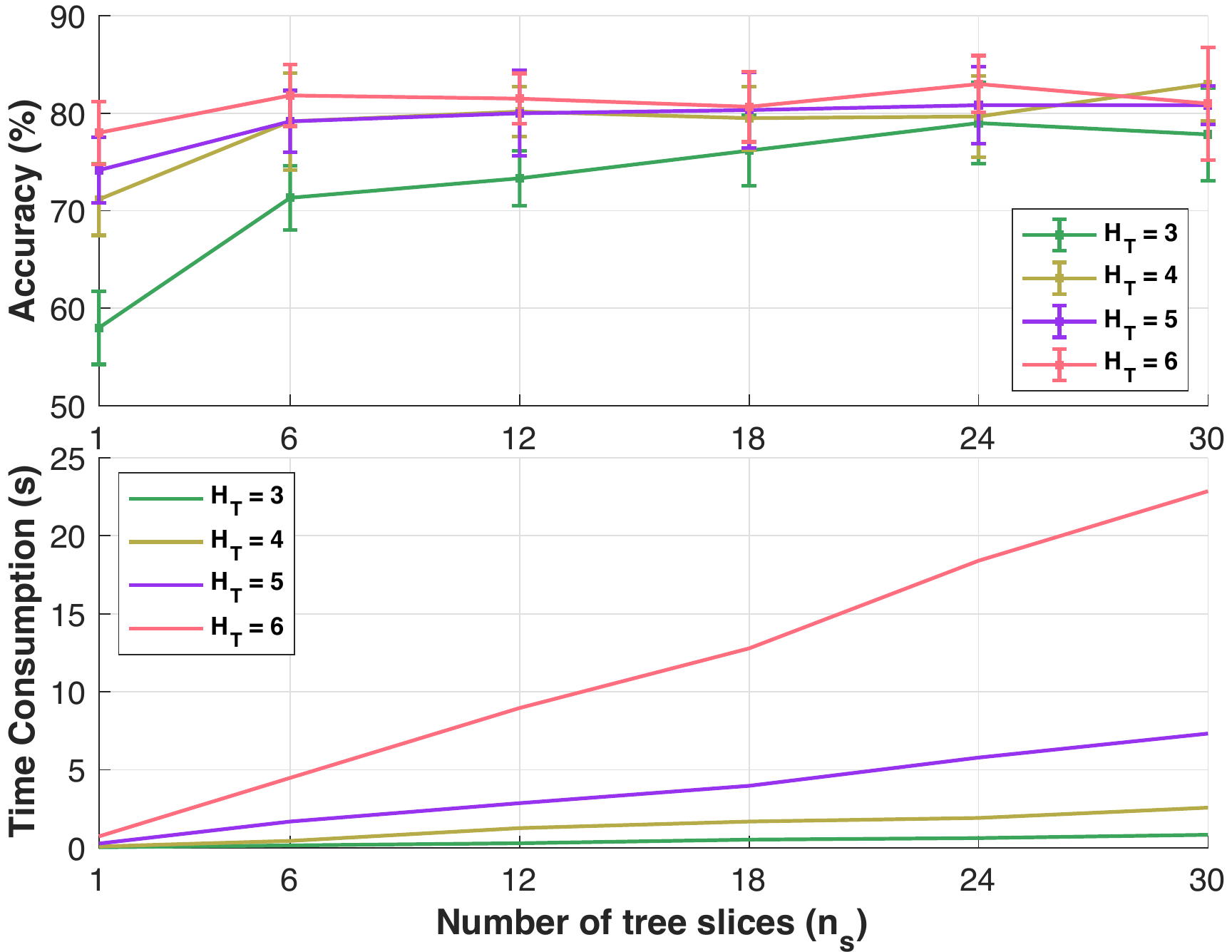}
  \end{center}
  \caption{Results of SVM and time consumption of kernel matrices of $k_{\TW}$ with different $(n_s, H_{\Tt})$ on MPEG7 dataset.}
  \label{fg:MPEG7KTW}
\end{figure}

\begin{figure}
  \begin{center}
    \includegraphics[width=0.5\textwidth]{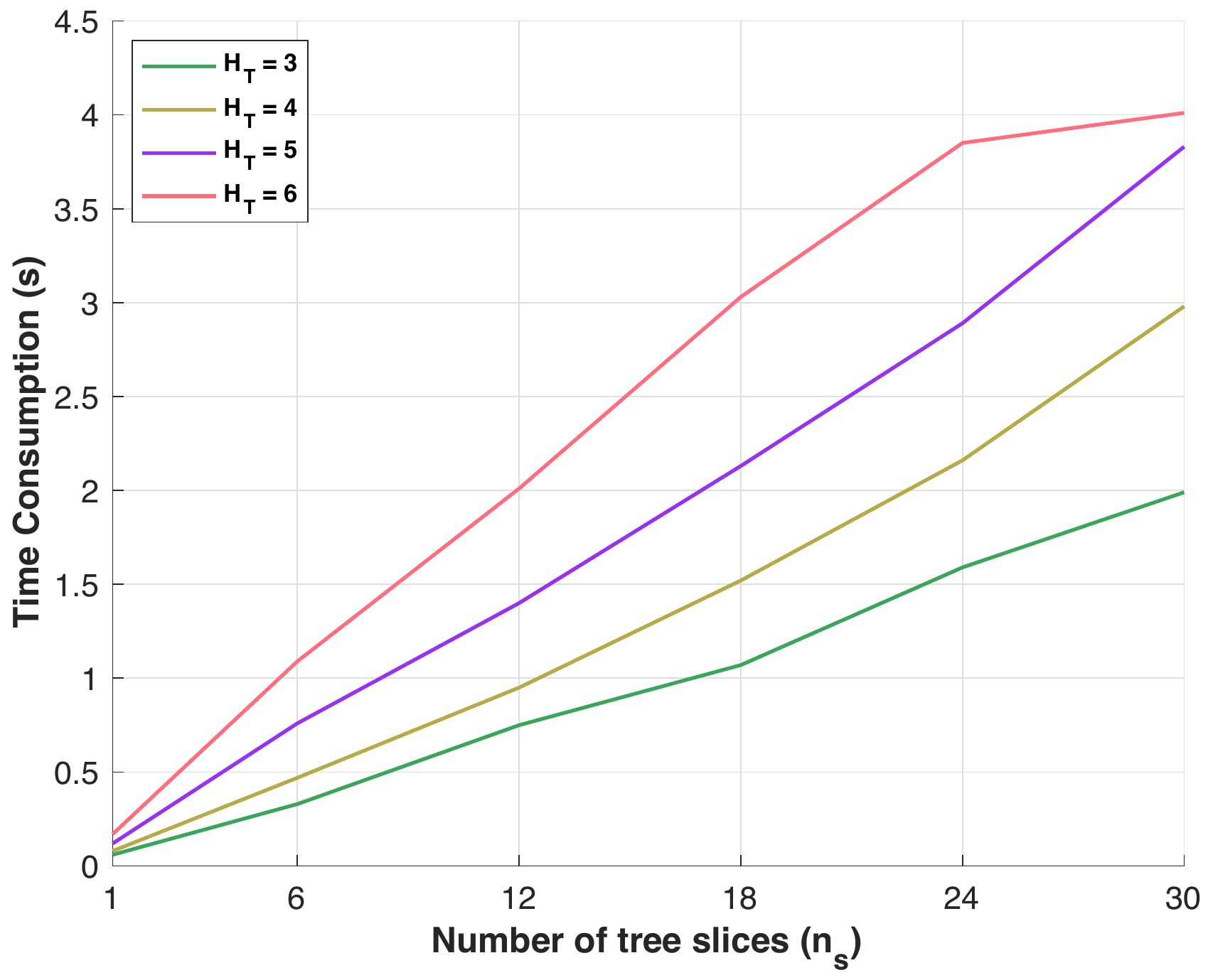}
  \end{center}
  \caption{Time consumption of kernel matrices of $k_{\TW}$ with different $(n_s, H_{\Tt})$ on granular packing system dataset.}
  \label{fg:GranularKTW}
\end{figure}

\begin{figure}
  \begin{center}
    \includegraphics[width=0.5\textwidth]{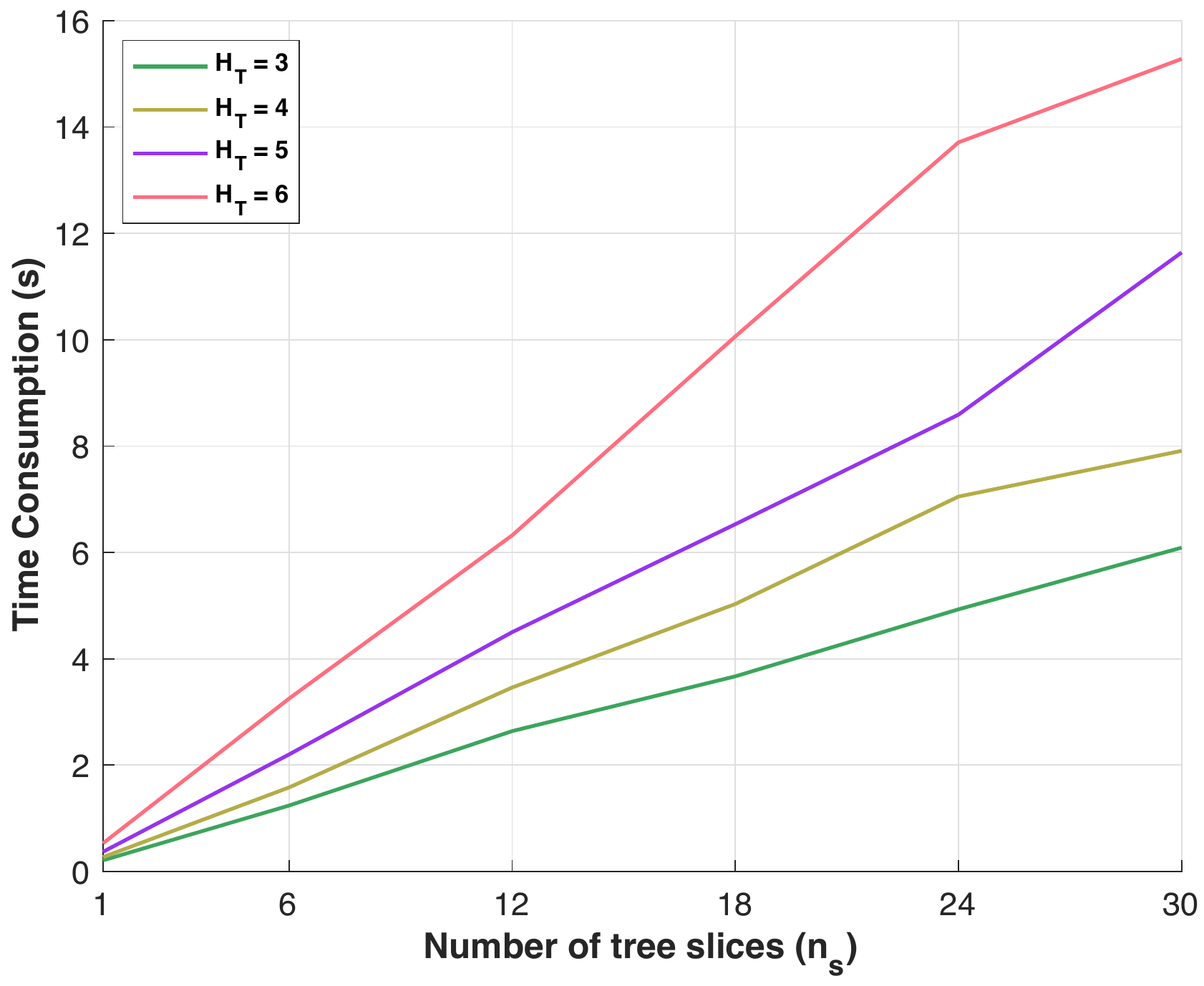}
  \end{center}
  \caption{Time consumption of kernel matrices of $k_{\TW}$ with different $(n_s, H_{\Tt})$ on SiO$_2$ dataset.}
  \label{fg:SiO2KTW}
\end{figure}

\begin{figure}
  \begin{center}
    \includegraphics[width=0.8\textwidth]{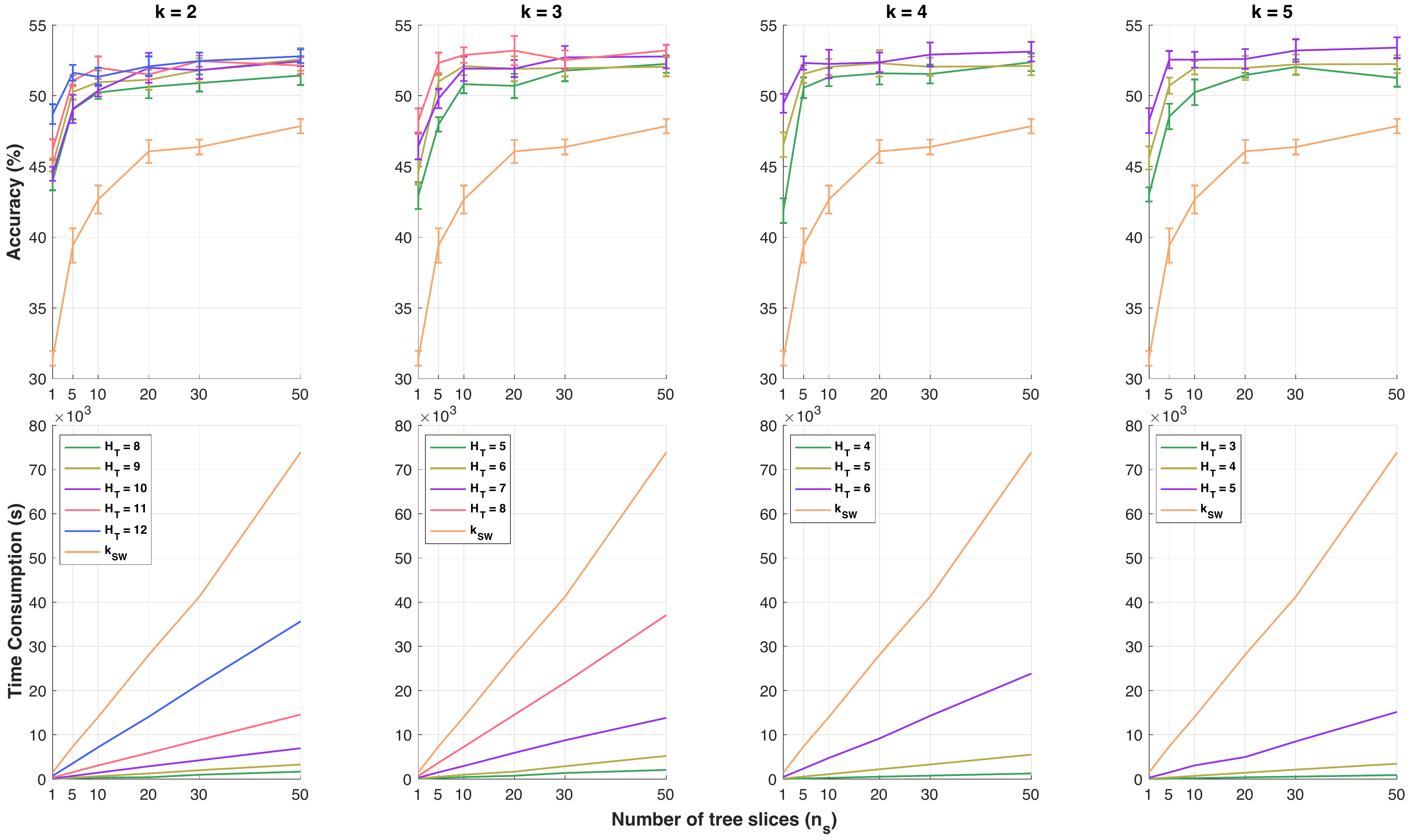}
  \end{center}
  \caption{Results of SVM and time consumption of kernel matrices of $k_{\TW}$ with different $(n_s, H_{\Tt}, \kappa)$, and $k_{\SW}$ with different $n_s$ on \texttt{RECIPE} dataset.}
  \label{fg:RECIPE_ALL_KTW}
\end{figure}

\begin{figure}
  \begin{center}
    \includegraphics[width=0.8\textwidth]{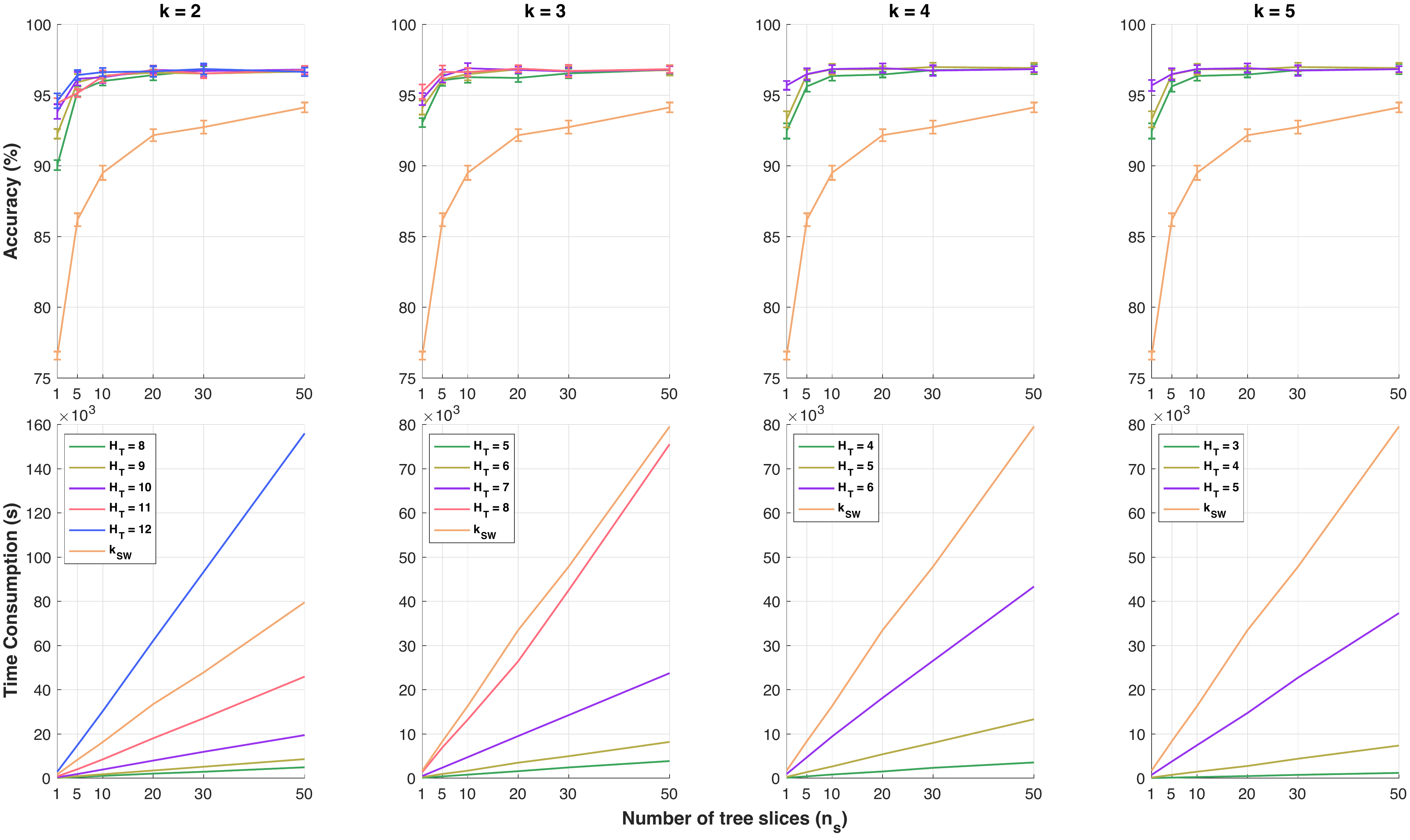}
  \end{center}
  \caption{Results of SVM and time consumption of kernel matrices of $k_{\TW}$ with different $(n_s, H_{\Tt}, \kappa)$, and $k_{\SW}$ with different $n_s$ on \texttt{CLASSIC} dataset.}
  \label{fg:CLASSIC_ALL_KTW}
\end{figure}

\begin{figure}
  \begin{center}
    \includegraphics[width=0.8\textwidth]{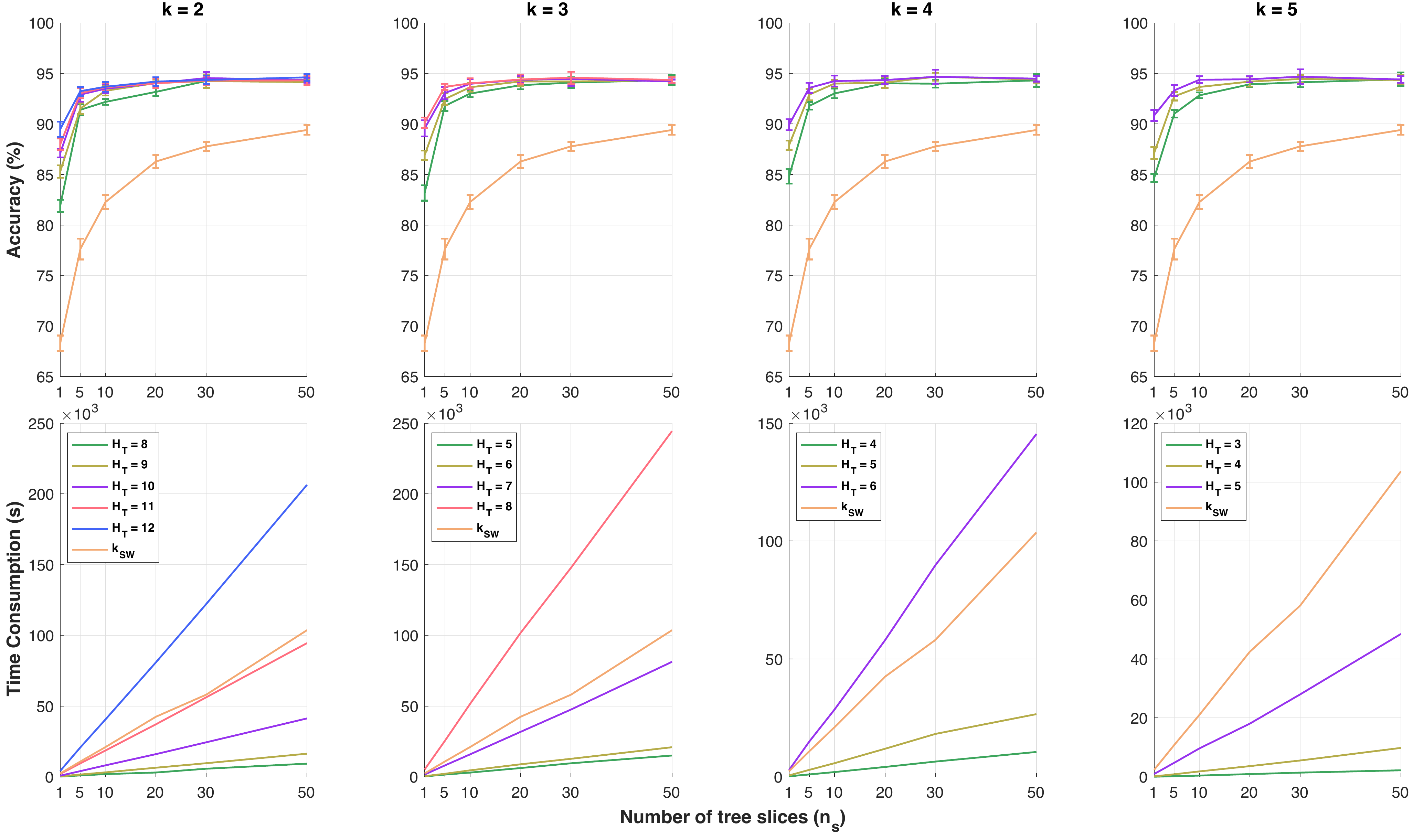}
  \end{center}
  \caption{Results of SVM and time consumption of kernel matrices of $k_{\TW}$ with different $(n_s, H_{\Tt}, \kappa)$, and $k_{\SW}$ with different $n_s$ on \texttt{AMAZON} dataset.}
  \label{fg:AMAZON_ALL_KTW}
\end{figure}

\section{Brief Reviews of Kernels, the Farthest-Point Clustering, and the Synthesized Orbit Dataset}
In this section, we give brief reviews for kernels, and the farthest-point clustering \cite{gonzalez1985clustering}.Then, we provide details for the synthesized orbit dataset for orbit recognition).

\subsection{A Brief Review of Kernels}

We review some important definitions and theorems about kernels used in our work.

\paragraph{Positive definite kernels \cite[p.66--67]{berg1984harmonic}.} A kernel function $k: \Xx \times \Xx \rightarrow \RR$ is  positive definite if $\forall n \in \NN^{*}, \forall x_1, x_2, ..., x_n \in \Xx$, $\sum_{i, j} c_i c_j k(x_i, x_j) \ge 0$, $\forall c_i \in \RR$.

\paragraph{Negative definite kernels \cite[p.66--67]{berg1984harmonic}.} A kernel function $k: \Xx \times \Xx \rightarrow \RR$ is  negative definite if $\forall n \ge 2, \forall x_1, x_2, ..., x_n \in \Xx$, $\sum_{i, j} c_i c_j k(x_i, x_j) \le 0$, $\forall c_i \in \RR$ such that $\sum_i c_i = 0$.

\paragraph{Berg-Christensen-Ressel Theorem \cite[Theorem 3.2.2, p.74]{berg1984harmonic}.}
If $\kappa$ is a \textit{negative definite} kernel, then kernel $k_{t}(x, z) := \exp{\left(- t \kappa(x, z)\right)}$ is positive definite for all $t > 0$.

\subsection{A Brief Review of the Farthest-Point Clustering for Data Space Partition}

The data space partition can be modeled as a $\kappa$-center problem. Given $n$ data points $x_1, x_2, ..., x_n$, and a predefined number of clusters $\kappa$, the goal of $\kappa$-center problem is to find a partition of $n$ points into $\kappa$ clusters $\texttt{S}_1, \texttt{S}_2, ..., \texttt{S}_{\kappa}$ as well as their corresponding centers $\texttt{c}_1, \texttt{c}_2, ..., \texttt{c}_{\kappa}$ to minimize the maximum radius of clusters.

The farthest-point clustering \cite{gonzalez1985clustering} is a simple greedy algorithm, summarized in Algorithm~\ref{alg:FarthestPointClustering}. Gonzalez \cite{gonzalez1985clustering} also proved that the farthest-point clustering computes a partition with maximum radius at most twice the optimum for $\kappa$-center clustering. The complexity for a direct implementation for the farthest-point clustering as in Algorithm~\ref{alg:FarthestPointClustering} is $O(n\kappa)$. This complexity can be reduced into $O(n\log \kappa)$ by using the algorithm in \cite{feder1988optimal}.

\begin{algorithm} 
\caption{\texttt{Farthest\_Point\_Clustering}($X, \kappa$)} 
\label{alg:FarthestPointClustering} 
\begin{algorithmic}[1] 
    \REQUIRE $X = \left(x_1, x_2, \dotsc, x_n \right)$: a set of $n$ data points, and $\kappa$: the predefined number of clusters.
    \ENSURE Clustering centers $\texttt{c}_1, \texttt{c}_2, ..., \texttt{c}_{\kappa}$ and cluster index for each point $x_i$.
    
     \STATE $\texttt{c}_1 \leftarrow$ a random point $x \in X$.
     \STATE Set of cluster $C \leftarrow \texttt{c}_1$.
     \STATE $i \leftarrow 1$.
     \WHILE{$i < \kappa$ and $n-i > 0$}
     	\STATE $i \leftarrow i + 1$.
     	\STATE $\texttt{c}_i \leftarrow \max_{x_j \in X} {\min_{\texttt{c} \in C} \norm{x_j - \texttt{c}}}$. (a farthest point $x_j \in X$ to $C$).
	\STATE $C \leftarrow C \cup \texttt{c}_i$. (Add the new center into $C$).
     \ENDWHILE   	
     \STATE Each data point $x_j \in X$ is assigned to its nearest center $\texttt{c}_i \in C$.
     
\end{algorithmic}
\end{algorithm}

\subsection{Details of the Synthesized Orbit Dataset}

Adams et al. \cite[\S 6.4.1]{adams2017persistence} proposed a synthesized dataset for link twist map, a discrete dynamical system to model flows in DNA microarrays \cite{hertzsch2007dna}. 

Given an initial position $(a_0, b_0) \in \left[0, 1\right]^2$, and $t>0$, an orbit is modeled as 
\begin{eqnarray}
a_{i+1} =& a_i + t b_i (1 - b_i) \mod 1, \\
b_{i+1} =& b_i + t a_{i+1} (1 - a_{i+1}) \mod 1.
\end{eqnarray}
There are $5$ classes, corresponding to $5$ different parameters $t = 2.5, 3.5, 4, 4.1, 4.3$. For each class, we generated $1000$ orbits with random initial positions where each orbit contains $1000$ points.

\begin{figure}
  \begin{center}
    \includegraphics[width=0.5\textwidth]{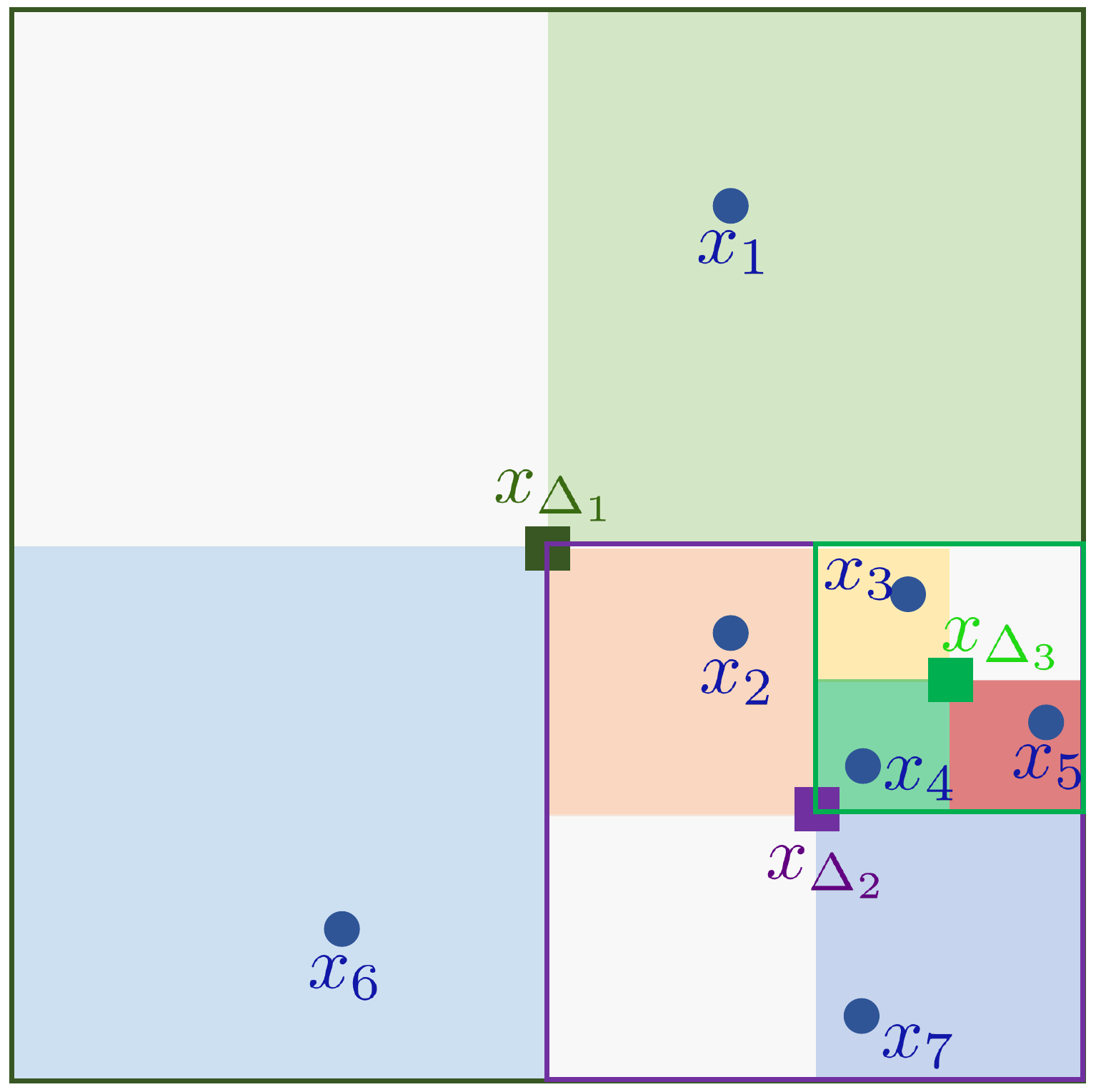}
  \end{center}
  \caption{An example about the partition-based tree metric for a set of points in a $2$-dimensional space.}
  \label{fg:PartitionTreeMetric_Point}
\end{figure}

\begin{figure}
  \begin{center}
    \includegraphics[width=0.7\textwidth]{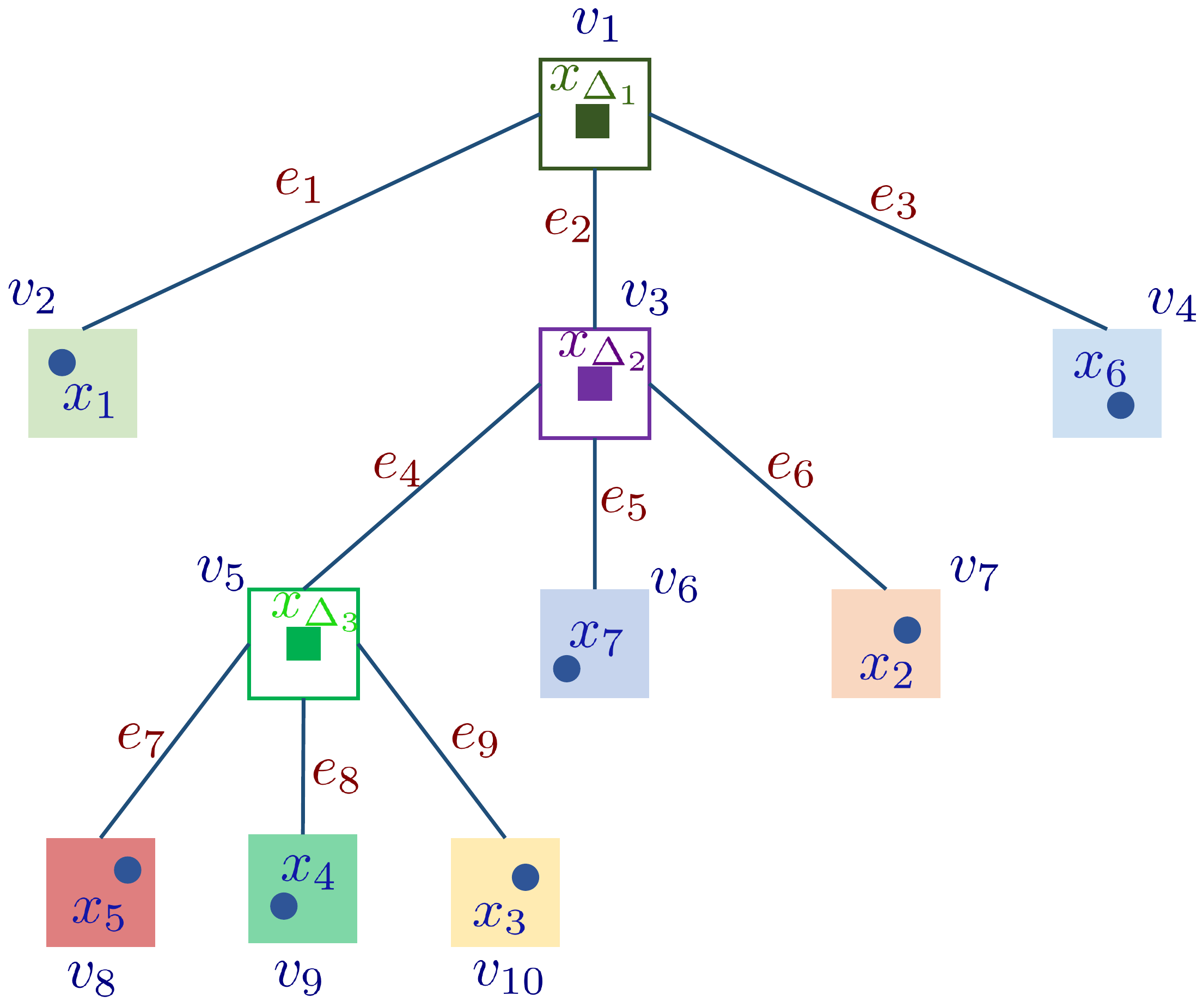}
  \end{center}
  \caption{The corresponding tree structure for the example in Figure~\ref{fg:PartitionTreeMetric_Point}.}
  \label{fg:PartitionTreeMetric_TM}
\end{figure}

\section{More Examples on the Partition-based Tree Metric, Quantization and Cluster Sensitivity Problems, and Persistence Diagrams}

In this section, we give some examples for the partition-based tree metric, quantization and cluster sensitivity problems and persistence diagrams.

\subsection{An Example on the Partition-based Tree Metric}

Given a set $X$ of $7$ data points $x_i \mid_{1 \le i \le 7}$ in a $2$-dimensional space as illustrated in Figure~\ref{fg:PartitionTreeMetric_Point}, one can choose a square region $\texttt{s}_0$ as the largest square in Figure~\ref{fg:PartitionTreeMetric_Point} containing all data points in $X$, and denote $\ell$ as the side of the largest square.

At height level $0$ in tree $\Tt$, applying the $\texttt{Partition\_HC}$ algorithm for $\texttt{s}_0$, one has $x_{\Delta_1}$ (center of $\texttt{s}_{0}$) as a node (i.e., the root) represented for $\texttt{s}_0$ in the constructed tree structure $\Tt$, and $4$ child square regions\footnote{we use a clock order to enumerate for those child square regions: top right --> bottom right --> bottom left --> top left.} with side $\ell/2$, denoted $\texttt{s}_{1a}$ (containing $x_1$), $\texttt{s}_{1b}$ (containing $x_2, x_3, x_4, x_5, x_7$), $\texttt{s}_{1c}$ (containing $x_6$), and $\texttt{s}_{1d}$ (containing no data points). Therefore, one can discard $\texttt{s}_{1d}$, use either data points ($x_1$, or $x_6$) or their centers represented for $\texttt{s}_{1a}$ and $\texttt{s}_{1c}$ respectively, and then apply the recursive procedure to partition for $\texttt{s}_{1b}$ (at height level $1$).

At height level $1$ in tree $\Tt$, applying the $\texttt{Partition\_HC}$ algorithm for $\texttt{s}_{1b}$, one has $x_{\Delta_2}$ (center of $\texttt{s}_{1b}$) as a node represented for $\texttt{s}_{1b}$ in the constructed tree structure $\Tt$, and $4$ child square regions with side $\ell/4$, denoted $\texttt{s}_{2a}$ (containing $x_3, x_4, x_5$), $\texttt{s}_{2b}$ (containing $x_7$), $\texttt{s}_{2c}$ (containing no data points), and $\texttt{s}_{2d}$ (containing $x_2$). Therefore, one can discard $\texttt{s}_{2c}$, use either data points ($x_7$, or $x_2$) or their centers represented for $\texttt{s}_{2b}$ and $\texttt{s}_{2d}$ respectively, and then apply the recursive procedure to partition for $\texttt{s}_{2a}$ (at height level 2).

At height level $2$ in tree $\Tt$, applying the $\texttt{Partition\_HC}$ algorithm for $\texttt{s}_{2a}$, one has $x_{\Delta_3}$ (center of $\texttt{s}_{2a}$) as a node represented for $\texttt{s}_{2a}$ in the constructed tree structure $\Tt$, and $4$ child square regions with side $\ell/8$, denoted $\texttt{s}_{3a}$ (containing no data points), $\texttt{s}_{3b}$ (containing $x_5$), $\texttt{s}_{3c}$ (containing $x_4$), and $\texttt{s}_{3d}$ (containing $x_3$). Therefore, one can discard $\texttt{s}_{3a}$, and use either data points ($x_5$, or $x_4$, or $x_3$) or their centers represented for $\texttt{s}_{3b}$, $\texttt{s}_{3c}$ and $\texttt{s}_{3d}$ respectively.

Hence, at the end, one obtains a tree structure $\Tt$ as illustrated in Figure~\ref{fg:PartitionTreeMetric_TM}, containing $10$ nodes $v_i \mid_{1 \le i \le 10}$, and $9$ edges $e_j \mid_{1 \le j \le 9}$. Node $x_1$ is the root of $\Tt$. The highest level in tree $\Tt$ is $3$. For lengths of edges in $\Tt$, one can apply any metrics in the $2$-dimensional space.

\subsection{Some Examples and Discussion about the Quantization and Cluster Sensitivity Problems}

The quantization or cluster sensitivity problem for partition or clustering is that some close data points are partitioned or clustered to adjacent, but different hypercubes or clusters respectively.

For example, in Figure~\ref{fg:Clustering2S}, we illustrate different results of clustering, obtained with different initializations for the farthest-point clustering for a given $10000$ random data points into $20$ clusters. For data points near a border of adjacent, but different clusters, although they are very close to each other, they are still in different clusters, or known as a cluster sensitivity problem. Whether those data points are clustered into the same or different cluster(s), it depends on an initialization of the farthest-point clustering. Therefore, by combining many different clustering results, obtained with various initializations for the farthest-point clustering algorithm, one can reduce an affect of the cluster sensitivity problem. Similarly for a quantization problem in the partition procedure (e.g. those data points near a border of adjacent, but different square regions of the same side in Figure~\ref{fg:PartitionTreeMetric_Point}).

\begin{figure}
  \begin{center}
    \includegraphics[width=0.82\textwidth]{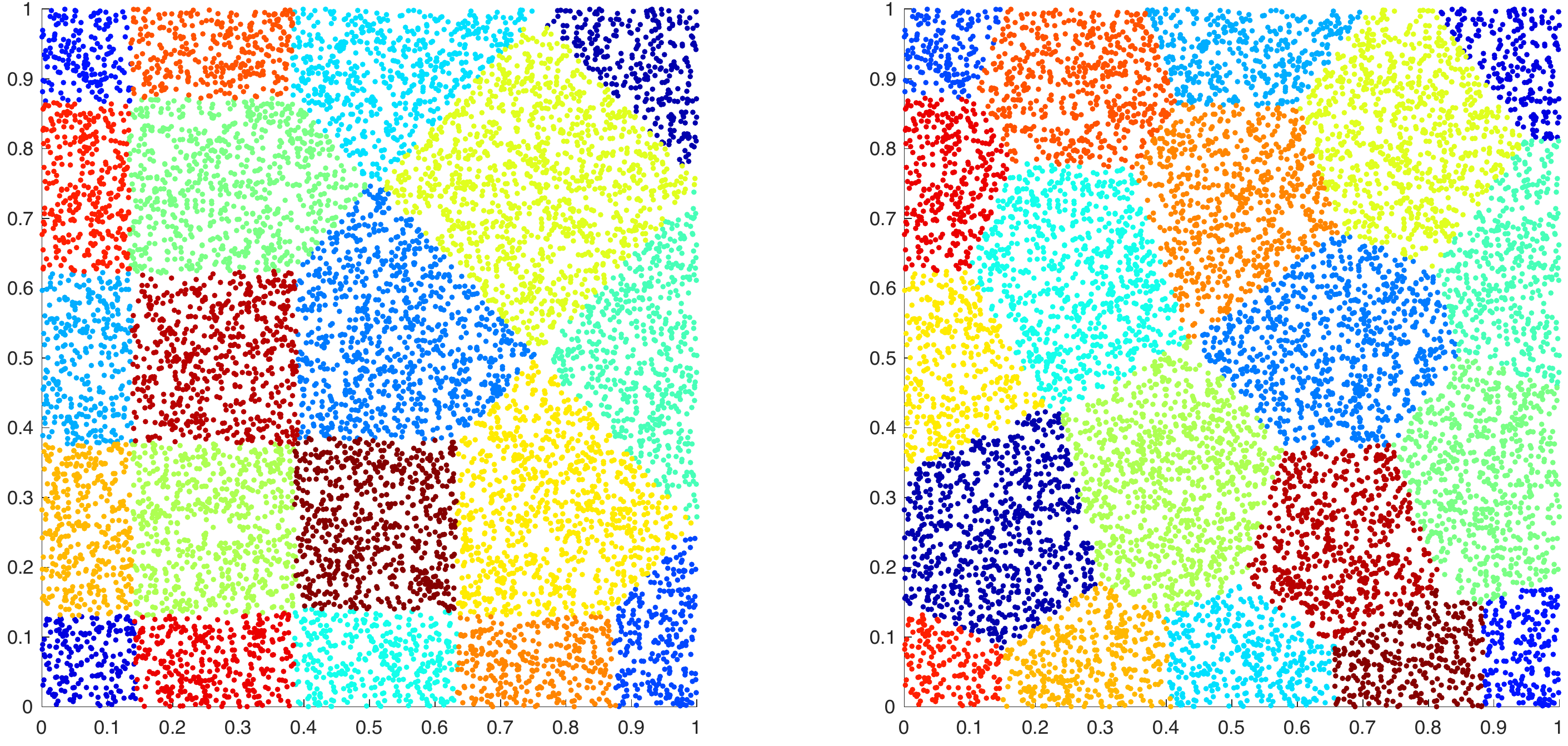}
  \end{center}
  \caption{An illustration of the farthest-point clustering for $10000$ data points into $20$ clusters with different initializations.}
  \label{fg:Clustering2S}
\end{figure}

\subsection{An Example of Persistence Diagrams}

In Figure \ref{fg:PDexample}, we give an example of a persistence diagram on a real-value function $f: \Xx \mapsto \RR$. Persistence homology considers a family of sublevel sets $f^{-1}((-\infty, t])$. When $t$ in $f^{-1}((-\infty, t])$ goes from $-\infty$ to $+\infty$, we collect all topological events, e.g., births and deaths of connected components (i.e., $0$-dimensional topological features). As in Figure \ref{fg:PDexample}, connected components appears (i.e., birth) at $t=t_1, t_2$, and disappear (i.e., death) at $t=+\infty, t_3$ respectively. Therefore, persistence diagram of $f$ is $\Dg f = \left\{(t_1, +\infty), (t_2, t_3) \right\}$.

\begin{figure}
  \begin{center}
    \includegraphics[width=0.55\textwidth]{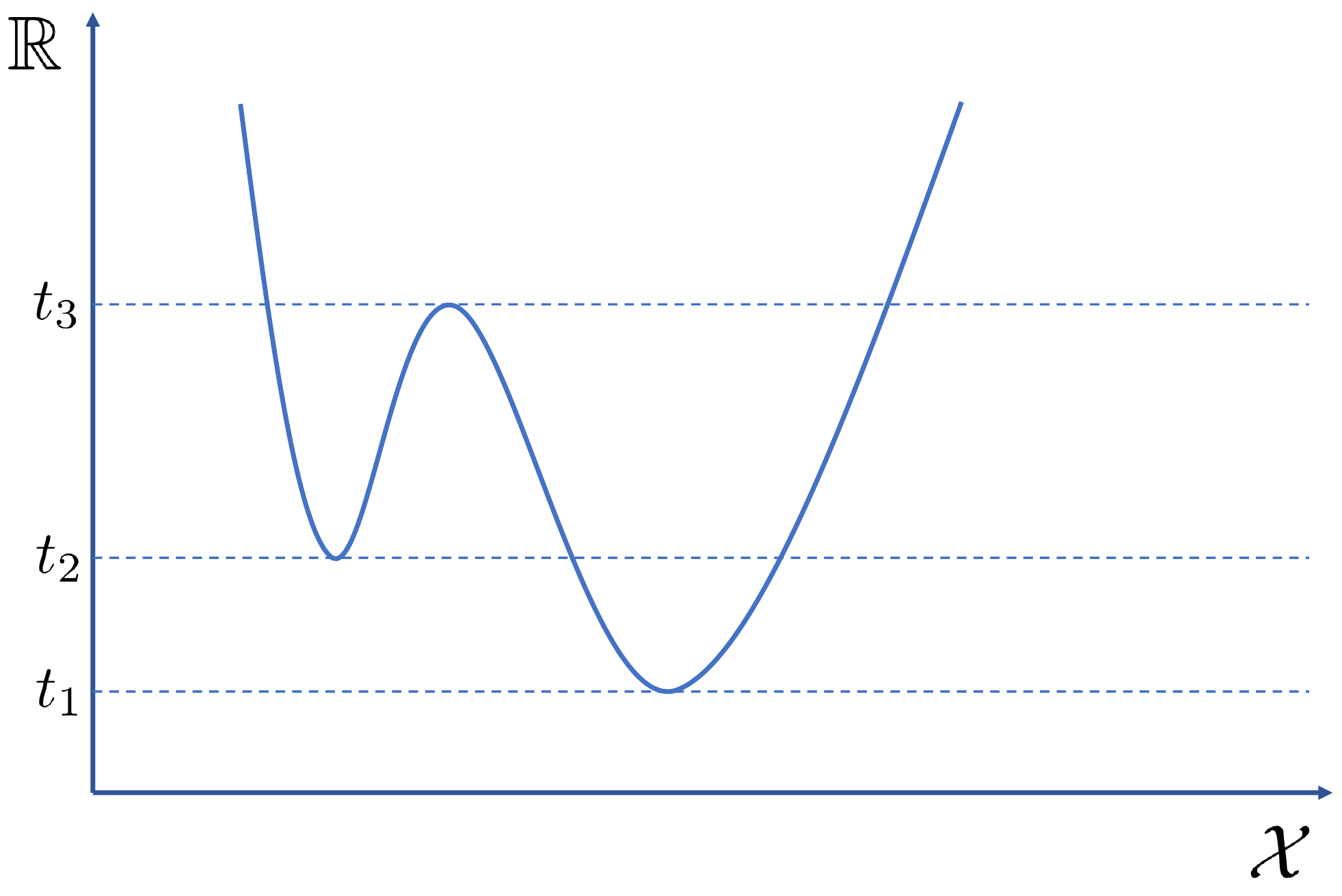}
  \end{center}
  \caption{An example of a persistence diagram on a real-value function $f: \Xx \mapsto \RR$. With sublevel sets $f^{-1}((-\infty, t])$ filtration, persistence diagram of $f$ is $\Dg f = \left\{(t_1, +\infty), (t_2, t_3) \right\}$.}
  \label{fg:PDexample}
\end{figure}

\bibliography{MLREF}

\begin{thebibliography}{10}

\bibitem{adams2017persistence}
H.~Adams, T.~Emerson, M.~Kirby, R.~Neville, C.~Peterson, P.~Shipman,
  S.~Chepushtanova, E.~Hanson, F.~Motta, and L.~Ziegelmeier.
\newblock Persistence images: A stable vector representation of persistent
  homology.
\newblock {\em The Journal of Machine Learning Research}, 18(1):218--252, 2017.

\bibitem{altschuler2018approximating}
J.~Altschuler, F.~Bach, A.~Rudi, and J.~Weed.
\newblock Approximating the quadratic transportation metric in near-linear
  time.
\newblock {\em arXiv preprint arXiv:1810.10046}, 2018.

\bibitem{altschuler2018massively}
J.~Altschuler, F.~Bach, A.~Rudi, and J.~Weed.
\newblock Massively scalable {S}inkhorn distances via the {N}ystrom method.
\newblock {\em arXiv preprint arXiv:1812.05189}, 2018.

\bibitem{altschuler2017near}
J.~Altschuler, J.~Weed, and P.~Rigollet.
\newblock Near-linear time approximation algorithms for optimal transport via
  {S}inkhorn iteration.
\newblock In {\em Advances in Neural Information Processing Systems}, pages
  1964--1974, 2017.

\bibitem{anonymous72}
Anonymous.
\newblock What is random packing?
\newblock {\em Nature}, 239:488--489, 1972.

\bibitem{backurs2019scalable}
A.~Backurs, P.~Indyk, K.~Onak, B.~Schieber, A.~Vakilian, and T.~Wagner.
\newblock Scalable fair clustering.
\newblock In {\em International Conference on Machine Learning}, pages
  405--413, 2019.

\bibitem{bartal1996probabilistic}
Y.~Bartal.
\newblock Probabilistic approximation of metric spaces and its algorithmic
  applications.
\newblock In {\em Proceedings of 37th Conference on Foundations of Computer
  Science}, pages 184--193, 1996.

\bibitem{bartal1998approximating}
Y.~Bartal.
\newblock On approximating arbitrary metrices by tree metrics.
\newblock In {\em STOC}, volume~98, pages 161--168, 1998.

\bibitem{berg1984harmonic}
C.~Berg, J.~P.~R. Christensen, and P.~Ressel.
\newblock {\em Harmonic analysis on semigroups}.
\newblock Springer-Verlag, 1984.

\bibitem{burkard1999}
R.~E. Burkard and E.~Cela.
\newblock Linear assignment problems and extensions.
\newblock In {\em Handbook of combinatorial optimization}, pages 75--149.
  Springer, 1999.

\bibitem{carriere17asliced}
M.~Carriere, M.~Cuturi, and S.~Oudot.
\newblock Sliced {W}asserstein kernel for persistence diagrams.
\newblock In {\em International Conference on Machine Learning}, volume~70,
  pages 664--673, 2017.

\bibitem{chang2011libsvm}
C.-C. Chang and C.-J. Lin.
\newblock Libsvm: a library for support vector machines.
\newblock {\em ACM transactions on intelligent systems and technology (TIST)},
  2(3):27, 2011.

\bibitem{charikar1998approximating}
M.~Charikar, C.~Chekuri, A.~Goel, S.~Guha, and S.~Plotkin.
\newblock Approximating a finite metric by a small number of tree metrics.
\newblock In {\em Proceedings 39th Annual Symposium on Foundations of Computer
  Science}, pages 379--388, 1998.

\bibitem{cuturi2013sinkhorn}
M.~Cuturi.
\newblock Sinkhorn distances: Lightspeed computation of optimal transport.
\newblock In {\em Advances in neural information processing systems}, pages
  2292--2300, 2013.

\bibitem{do2011sublinear}
K.~Do~Ba, H.~L. Nguyen, H.~N. Nguyen, and R.~Rubinfeld.
\newblock Sublinear time algorithms for {E}arth {M}over's distance.
\newblock {\em Theory of Computing Systems}, 48(2):428--442, 2011.

\bibitem{dong2019scalable}
Y.~Dong, P.~Indyk, I.~Razenshteyn, and T.~Wagner.
\newblock Scalable nearest neighbor search for optimal transport.
\newblock {\em arXiv preprint arXiv:1910.04126}, 2019.

\bibitem{dvurechensky18}
P.~Dvurechensky, A.~Gasnikov, and A.~Kroshnin.
\newblock Computational optimal transport: Complexity by accelerated gradient
  descent is better than by {S}inkhorn's algorithm.
\newblock In {\em Proceedings of the 35th International Conference on Machine
  Learning}, pages 1367--1376, 2018.

\bibitem{ebert2017construction}
J.~Ebert, V.~Spokoiny, and A.~Suvorikova.
\newblock Construction of non-asymptotic confidence sets in 2-{W}asserstein
  space.
\newblock {\em arXiv preprint arXiv:1703.03658}, 2017.

\bibitem{edelsbrunner2008persistent}
H.~Edelsbrunner and J.~Harer.
\newblock Persistent homology - a survey.
\newblock {\em Contemporary mathematics}, 453:257--282, 2008.

\bibitem{elliott1983physics}
S.~R. Elliott.
\newblock Physics of amorphous materials.
\newblock {\em Longman Group}, 1983.

\bibitem{evans2012phylogenetic}
S.~N. Evans and F.~A. Matsen.
\newblock The phylogenetic {K}antorovich--{R}ubinstein metric for environmental
  sequence samples.
\newblock {\em Journal of the Royal Statistical Society: Series B (Statistical
  Methodology)}, 74(3):569--592, 2012.

\bibitem{fakcharoenphol2004tight}
J.~Fakcharoenphol, S.~Rao, and K.~Talwar.
\newblock A tight bound on approximating arbitrary metrics by tree metrics.
\newblock {\em Journal of Computer and System Sciences}, 69(3):485--497, 2004.

\bibitem{feder1988optimal}
T.~Feder and D.~Greene.
\newblock Optimal algorithms for approximate clustering.
\newblock In {\em Proceedings of the twentieth annual ACM symposium on Theory
  of computing}, pages 434--444. ACM, 1988.

\bibitem{francois2013geometrical}
N.~Francois, M.~Saadatfar, R.~Cruikshank, and A.~Sheppard.
\newblock Geometrical frustration in amorphous and partially crystallized
  packings of spheres.
\newblock {\em Physical review letters}, 111(14):148001, 2013.

\bibitem{franklin1989scaling}
J.~Franklin and J.~Lorenz.
\newblock On the scaling of multidimensional matrices.
\newblock {\em Linear Algebra and its applications}, 114:717--735, 1989.

\bibitem{genevay2016stochastic}
A.~Genevay, M.~Cuturi, G.~Peyre, and F.~Bach.
\newblock Stochastic optimization for large-scale optimal transport.
\newblock In {\em Advances in Neural Information Processing Systems}, pages
  3440--3448, 2016.

\bibitem{gonzalez1985clustering}
T.~F. Gonzalez.
\newblock Clustering to minimize the maximum intercluster distance.
\newblock {\em Theoretical Computer Science}, 38:293--306, 1985.

\bibitem{harchaoui2009kernel}
Z.~Harchaoui, E.~Moulines, and F.~R. Bach.
\newblock Kernel change-point analysis.
\newblock In {\em Advances in neural information processing systems}, pages
  609--616, 2009.

\bibitem{hertzsch2007dna}
J.-M. Hertzsch, R.~Sturman, and S.~Wiggins.
\newblock Dna microarrays: design principles for maximizing ergodic, chaotic
  mixing.
\newblock {\em Small}, 3(2):202--218, 2007.

\bibitem{indyk2001algorithmic}
P.~Indyk.
\newblock Algorithmic applications of low-distortion geometric embeddings.
\newblock In {\em Proceedings 42nd IEEE Symposium on Foundations of Computer
  Science}, pages 10--33, 2001.

\bibitem{indyk2017practical}
P.~Indyk, I.~Razenshteyn, and T.~Wagner.
\newblock Practical data-dependent metric compression with provable guarantees.
\newblock In {\em Advances in Neural Information Processing Systems}, pages
  2617--2626, 2017.

\bibitem{indyk2003fast}
P.~Indyk and N.~Thaper.
\newblock Fast image retrieval via embeddings.
\newblock {\em International Workshop on Statistical and Computational Theories
  of Vision}, 2003.

\bibitem{johnson1967hierarchical}
S.~C. Johnson.
\newblock Hierarchical clustering schemes.
\newblock {\em Psychometrika}, 32(3):241--254, 1967.

\bibitem{kantorovich1942}
L.~V. Kantorovich.
\newblock On the transfer of masses.
\newblock In {\em Dokl. Akad. Nauk. SSSR}, volume~37, pages 227--229, 1942.

\bibitem{kloeckner2015geometric}
B.~R. Kloeckner.
\newblock A geometric study of {W}asserstein spaces: ultrametrics.
\newblock {\em Mathematika}, 61(1):162--178, 2015.

\bibitem{kolouri2016sliced}
S.~Kolouri, Y.~Zou, and G.~K. Rohde.
\newblock Sliced {W}asserstein kernels for probability distributions.
\newblock In {\em Proceedings of the IEEE Conference on Computer Vision and
  Pattern Recognition (CVPR)}, pages 5258--5267, 2016.

\bibitem{kriege2016valid}
N.~M. Kriege, P.-L. Giscard, and R.~Wilson.
\newblock On valid optimal assignment kernels and applications to graph
  classification.
\newblock In {\em Advances in Neural Information Processing Systems}, pages
  1623--1631, 2016.

\bibitem{kusano2017kernelJMLR}
G.~Kusano, K.~Fukumizu, and Y.~Hiraoka.
\newblock Kernel method for persistence diagrams via kernel embedding and
  weight factor.
\newblock {\em Journal of Machine Learning Research}, 18(189):1--41, 2018.

\bibitem{kusner2015word}
M.~Kusner, Y.~Sun, N.~Kolkin, and K.~Weinberger.
\newblock From word embeddings to document distances.
\newblock In {\em International Conference on Machine Learning}, pages
  957--966, 2015.

\bibitem{latecki2000shape}
L.~J. Latecki, R.~Lakamper, and T.~Eckhardt.
\newblock Shape descriptors for non-rigid shapes with a single closed contour.
\newblock In {\em Proceedings of the IEEE Conference on Computer Vision and
  Pattern Recognition (CVPR)}, volume~1, pages 424--429, 2000.

\bibitem{lavenant2018dynamical}
H.~Lavenant, S.~Claici, E.~Chien, and J.~Solomon.
\newblock Dynamical optimal transport on discrete surfaces.
\newblock In {\em SIGGRAPH Asia 2018 Technical Papers}, page 250. ACM, 2018.

\bibitem{le2018persistence}
T.~Le and M.~Yamada.
\newblock Persistence {F}isher kernel: A {R}iemannian manifold kernel for
  persistence diagrams.
\newblock In {\em Advances in Neural Information Processing Systems}, pages
  10028--10039, 2018.

\bibitem{LeeNIPS2018}
J.~Lee and M.~Raginsky.
\newblock Minimax statistical learning with {W}asserstein distances.
\newblock In {\em Advances in Neural Information Processing Systems}, pages
  2692--2701, 2018.

\bibitem{lozupone2005unifrac}
C.~Lozupone and R.~Knight.
\newblock Unifrac: a new phylogenetic method for comparing microbial
  communities.
\newblock {\em Applied and environmental microbiology}, 71(12):8228--8235,
  2005.

\bibitem{lozupone2007quantitative}
C.~A. Lozupone, M.~Hamady, S.~T. Kelley, and R.~Knight.
\newblock Quantitative and qualitative $\beta$ diversity measures lead to
  different insights into factors that structure microbial communities.
\newblock {\em Applied and environmental microbiology}, 73(5):1576--1585, 2007.

\bibitem{mcgregor2013sketching}
A.~McGregor and D.~Stubbs.
\newblock Sketching {E}arth-{M}over distance on graph metrics.
\newblock In {\em Approximation, Randomization, and Combinatorial Optimization.
  Algorithms and Techniques}, pages 274--286. Springer, 2013.

\bibitem{mikolov2013distributed}
T.~Mikolov, I.~Sutskever, K.~Chen, G.~S. Corrado, and J.~Dean.
\newblock Distributed representations of words and phrases and their
  compositionality.
\newblock In {\em Advances in neural information processing systems}, pages
  3111--3119, 2013.

\bibitem{morariu2009automatic}
V.~I. Morariu, B.~V. Srinivasan, V.~C. Raykar, R.~Duraiswami, and L.~S. Davis.
\newblock Automatic online tuning for fast gaussian summation.
\newblock In {\em Advances in neural information processing systems}, pages
  1113--1120, 2009.

\bibitem{nakamura2015persistent}
T.~Nakamura, Y.~Hiraoka, A.~Hirata, E.~G. Escolar, and Y.~Nishiura.
\newblock Persistent homology and many-body atomic structure for medium-range
  order in the glass.
\newblock {\em Nanotechnology}, 26(30):304001, 2015.

\bibitem{panaretos2016amplitude}
V.~M. Panaretos, Y.~Zemel, et~al.
\newblock Amplitude and phase variation of point processes.
\newblock {\em The Annals of Statistics}, 44(2):771--812, 2016.

\bibitem{pele2009fast}
O.~Pele and M.~Werman.
\newblock Fast and robust {E}arth {M}over's distances.
\newblock In {\em International Conference on Computer Vision}, pages 460--467,
  2009.

\bibitem{PeyreCuturiBook}
G.~Peyr\'e and M.~Cuturi.
\newblock Computational optimal transport.
\newblock {\em Foundations and Trends in Machine Learning}, 11(5-6):355--607,
  2019.

\bibitem{reininghaus2015stable}
J.~Reininghaus, S.~Huber, U.~Bauer, and R.~Kwitt.
\newblock A stable multi-scale kernel for topological machine learning.
\newblock In {\em Proceedings of the IEEE conference on computer vision and
  pattern recognition (CVPR)}, pages 4741--4748, 2015.

\bibitem{rubner2000}
Y.~Rubner, C.~Tomasi, and L.~J. Guibas.
\newblock The {E}arth {M}over's distance as a metric for image retrieval.
\newblock {\em International journal of computer vision}, 40(2):99--121, 2000.

\bibitem{salton1988term}
G.~Salton and C.~Buckley.
\newblock Term-weighting approaches in automatic text retrieval.
\newblock {\em Information processing \& management}, 24(5):513--523, 1988.

\bibitem{samet1984quadtree}
H.~Samet.
\newblock The quadtree and related hierarchical data structures.
\newblock {\em ACM Computing Surveys (CSUR)}, 16(2):187--260, 1984.

\bibitem{SantambrogioBook}
F.~Santambrogio.
\newblock {\em Optimal transport for applied mathematicians}.
\newblock Birkhauser, 2015.

\bibitem{semple2003phylogenetics}
C.~Semple and M.~Steel.
\newblock Phylogenetics.
\newblock {\em Oxford Lecture Series in Mathematics and its Applications},
  2003.

\bibitem{shkarin2004isometric}
S.~A. Shkarin.
\newblock Isometric embedding of finite ultrametric spaces in banach spaces.
\newblock {\em Topology and its Applications}, 142(1-3):13--17, 2004.

\bibitem{csimcsekli2018sliced}
U.~{\c{S}}im{\c{s}}ekli, A.~Liutkus, S.~Majewski, and A.~Durmus.
\newblock Sliced-{W}asserstein flows: Nonparametric generative modeling via
  optimal transport and diffusions.
\newblock {\em arXiv preprint arXiv:1806.08141}, 2018.

\bibitem{solomon2015convolutional}
J.~Solomon, F.~De~Goes, G.~Peyre, M.~Cuturi, A.~Butscher, A.~Nguyen, T.~Du, and
  L.~Guibas.
\newblock Convolutional wasserstein distances: Efficient optimal transportation
  on geometric domains.
\newblock {\em ACM Transactions on Graphics (TOG)}, 34(4):66, 2015.

\bibitem{sommerfeld2018inference}
M.~Sommerfeld and A.~Munk.
\newblock Inference for empirical {W}asserstein distances on finite spaces.
\newblock {\em Journal of the Royal Statistical Society: Series B (Statistical
  Methodology)}, 80(1):219--238, 2018.

\bibitem{turner2014persistent}
K.~Turner, S.~Mukherjee, and D.~M. Boyer.
\newblock Persistent homology transform for modeling shapes and surfaces.
\newblock {\em Information and Inference: A Journal of the IMA}, 3(4):310--344,
  2014.

\bibitem{villani2003topics}
C.~Villani.
\newblock {\em Topics in optimal transportation}.
\newblock American Mathematical Soc., 2003.

\bibitem{villani2008optimal}
C.~Villani.
\newblock {\em Optimal transport: old and new}, volume 338.
\newblock Springer Science and Business Media, 2008.

\bibitem{yang2005efficient}
C.~Yang, R.~Duraiswami, and L.~S. Davis.
\newblock Efficient kernel machines using the improved fast gauss transform.
\newblock In {\em Advances in neural information processing systems}, pages
  1561--1568, 2005.

\end{thebibliography}



\end{document}